\documentclass[dvipsnames,format=sigconf,anonymous=false,review=false]{acmart}
\renewcommand\footnotetextcopyrightpermission[1]{}
\AtBeginDocument{%
  }

\usepackage{microtype}
\usepackage{subcaption}
\usepackage{booktabs}
\usepackage{float}
\usepackage{bm}
\usepackage{algorithm}
\usepackage{algorithmic}
\usepackage{multirow}
\usepackage{mathtools}
\usepackage[capitalize,noabbrev]{cleveref}
\usepackage{todonotes}
\usepackage{placeins}

\acmConference[]{}{}{} 
\acmBooktitle{}
\begin{document}

\title{Non-linear PCA via Evolution Strategies: a Novel Objective Function}


\author{Thomas Uriot}
\affiliation{%
  \institution{}
  \country{The Netherlands}}
\email{uriot.thomas@gmail.com}

\author{Elise Chung}
\affiliation{%
  \institution{}
  \country{The Netherlands}}
\email{elise.chungg@gmail.com}

\begin{abstract}
Principal Component Analysis (PCA) is a powerful and popular dimensionality reduction technique. However, due to its linear nature, it often fails to capture the complex underlying structure of real-world data. While Kernel PCA (kPCA) addresses non-linearity, it sacrifices interpretability and struggles with hyperparameter selection. In this paper, we propose a robust non-linear PCA framework that unifies the interpretability of PCA with the flexibility of neural networks. Our method parametrizes variable transformations via neural networks, optimized using Evolution Strategies (ES) to handle the non-differentiability of eigendecomposition. We introduce a novel, granular objective function that maximizes the individual variance contribution of each variable providing a stronger learning signal than global variance maximization. This approach natively handles categorical and ordinal variables without the dimensional explosion associated with one-hot encoding. We demonstrate that our method significantly outperforms both linear PCA and kPCA in explained variance across synthetic and real-world datasets. At the same time, it preserves PCA’s interpretability, enabling visualization and analysis of feature contributions using standard tools such as biplots. The code can be found on GitHub\footnote{\url{https://github.com/pinouche/nonlinear-PCA}}.
\end{abstract}





\maketitle

\section{Introduction}

PCA \cite{pearson1901liii} is a widely used dimensionality reduction technique that seeks a set of orthogonal linear transformations of the original data for which variance is maximized. Arguably, its popularity is mainly due to its deterministic nature and the interpretability of its solution. The latter can be attributed to three factors: (i) the orthogonality of the principal components (PCs), (ii) the linear relationship between the variables, and (iii) the explained variance associated to each PC which provides us with an intuitive understanding of the importance of each transformation and how easily (linearly) reducible a dataset can be. However, an obvious drawback of PCA is that it can only uncover linear relationships between the original variables. In other words, PCA treats potentially nonlinear structure in the data as a lack of structure altogether.

\FloatBarrier

\begin{figure}[H]
    \centering
    \includegraphics[width=1.0\linewidth]{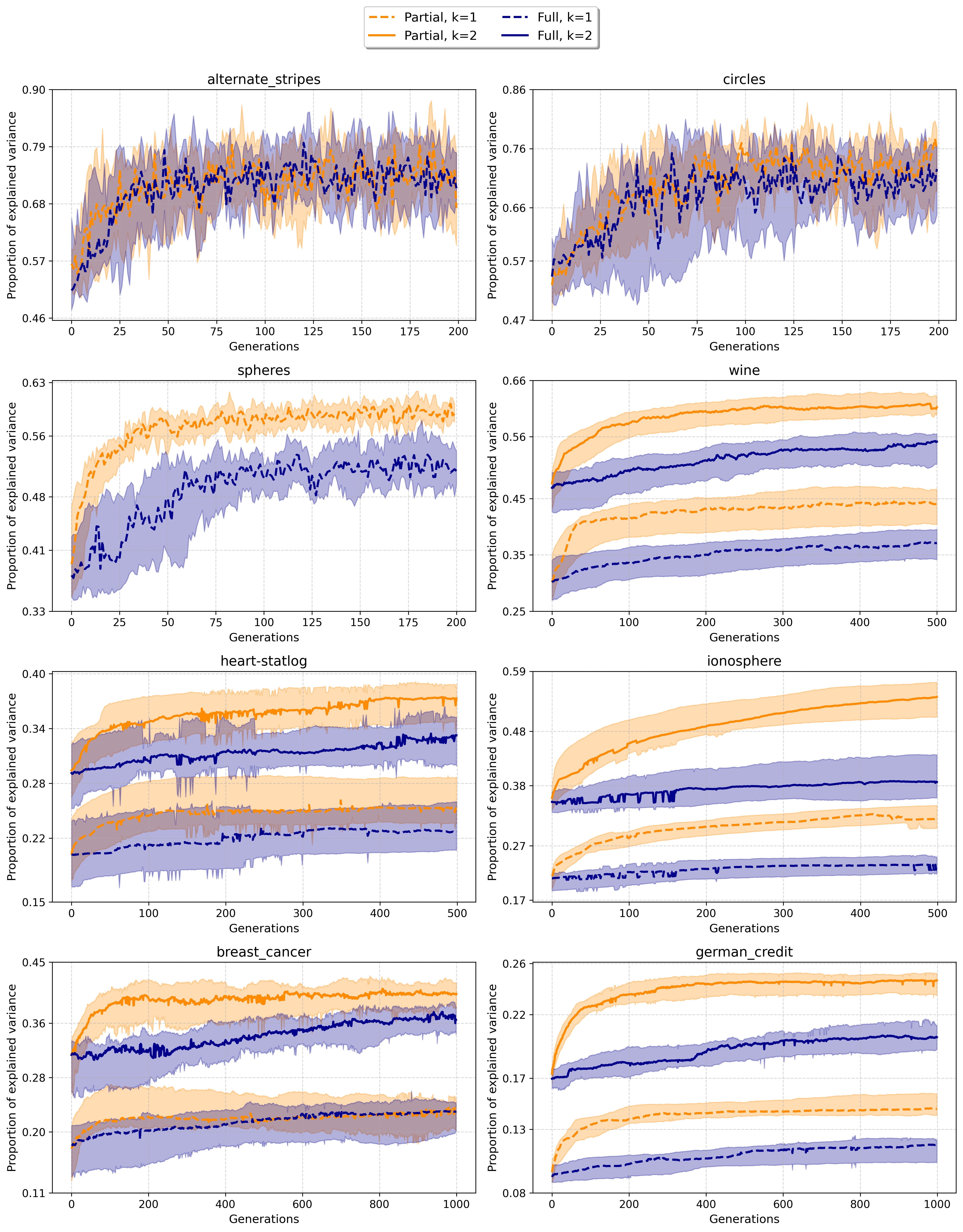}
    \caption{Proportion of explained variance across eight datasets on the validation set (0.75/0.25 split). Lines display the median over 15 independent runs while shaded regions represent the 20th and 80th percentiles. The orange curves illustrate the explained variance achieved using the novel \textit{partial} contribution objective (Equation \ref{eqn:contrib_objective}), where individual variable contributions are optimized separately. The blue curves show the variance achieved using the \textit{total} explained variance objective (Equation \ref{eqn:total_variance_obj}). Dashed lines ($k=1$) represent the proportion of variance accounted for by the first eigenvalue alone, while the solid lines ($k=2$) represent the cumulative variance of the first two eigenvalues. The value of $k$ is also impacting the objectives defined in Equation \ref{eqn:contrib_objective} and Equation \ref{eqn:total_variance_obj} where the optimization is targeted to maximize the variance up to the $k^{th}$ eigenvalue without consideration for subsequent components. Note that for the \textit{alternate\_stripes}, \textit{circles}, and \textit{spheres} datasets, only the $k=1$ case is considered. Across all benchmarks, the \textit{partial} contribution method consistently outperforms the overall variance objective, achieving higher explained variance and faster convergence.}
    \label{fig:training_curves}
\end{figure}

A number of non-linear PCA (NLPCA) methods have been proposed to tackle the aforementioned non-linearity issue. Generally, NLPCA methods can be grouped into three main research branches, as described in a review \cite{kruger2008developments}. Before describing each research branch, one should note that the labeling as NLPCA can be misleading and confusing as the two first branches cannot be considered as natural non-linear extensions of PCA.
The first branch focuses on principal curves and manifolds \cite{hastie1989principal}, which minimize the sum of squared projections onto a non-linear curve. While conceptually elegant, this method is computationally expensive ($O(n^2)$ complexity) and does not yield a model that can be used to make predictions on new data. More importantly, principal manifolds do not offer the same intrinsic variance maximization properties as in linear PCA, lacking interpretability. The second main line of research in NLPCA is the use of auto-associative neural networks \cite{kramer1991nonlinear, tan1995reducing} (i.e., auto-encoders), where the bottleneck layer can be interpreted as the resulting lower-dimensional representation of the original data. In the special case of linear activation functions and appropriately chosen encoding and decoding weight matrices, such networks reduce exactly to PCA. Moving beyond this linear regime, however, the learned latent variables generally do not retain the defining properties of linear PCA (variance maximization, orthogonality of the components, or easily interpretable relationships with the original variables). Furthermore, training such networks may be difficult in low-data regimes. The third and final area of research in NLPCA consists in a non-linear extension of PCA, called kernel PCA (kPCA) \cite{scholkopf1998nonlinear}. This method aims at first computing the kernelized Gram matrix and performing linear PCA on it. It is argued that kPCA is the only true extension of linear PCA \cite{kruger2008developments}, as it inherits all the properties of linear PCA, and, in the case of a linear mapping, is equivalent to linear PCA. In kPCA, however, it is difficult to interpret the solution as the principal components do not relate to the transformed variables in the feature space. In addition, choosing the right kernel function and correctly setting its hyperparameters is more of an art than a science. In fact, in an extensive review of NLPCA \cite{kruger2008developments}, the authors argue that having a more general framework with parameterized transformations has been understudied and would be a useful contribution to the field of NLPCA. This is the gap that we aim to address in this work.

Finally, a further major drawback of PCA and the existing NLPCA methods is their inability to deal with categorical or ordinal variables. Indeed, while one could still convert a categorical column to a one-hot encoding format, in the resulting vector space, each level (subcategory) would be orthogonal and equidistant from each other. However, it would make sense for some levels to be closer to each other in terms of distance metric, depending on what their correlations with other variables are, for a given dataset. In the case of non-curated, real-world data, levels representing the same concept can be recorded differently (e.g., \textit{restaurant} and \textit{eatery}), leading to an explosion of the number of unique levels \cite{cerda2018similarity}. This becomes prohibitive in PCA, as each one-hot encoded column would contribute to one unit of variance, thus dominating the variance contribution of numerical variables.

Thus, on top of the non-linear element, one of the main motivations and contributions of this work is to propose a method that can natively deal with numerical, categorical, and ordinal variables simultaneously. Our method draws from several classical statistical methods such as PCA, Correspondence Analysis (CA) \cite{hirschfeld1935connection}, and Multiple Factor Analysis (MFA) \cite{thurstone1931multiple, escofier1998analyses}. In summary, the main contributions of this paper are as follows.

\begin{itemize}
    \item \textit{Non-Linear Extension via Neuroevolution}: We propose a robust non-linear PCA framework where variable transformations are parameterized by neural networks and optimized using Evolution Strategies. This allows for arbitrary complex mappings while bypassing the differentiability constraints of standard eigendecomposition.
    \item \textit{Granular Variance Maximization}: We introduce a novel \textit{partial} objective function that decomposes global explained variance to maximize the specific contribution of each variable individually. We demonstrate that this granular learning signal significantly outperforms standard global variance maximization, particularly in higher dimensions.
    \item \textit{Unified Data Handling}: Our method natively handle categorical, ordinal, and numerical variables without the dimensional explosion associated with one-hot encoding. This effectively bridges the gap between Kernel PCA (which struggles with categories) and traditional component-based methods like MFA (which assume linearity).
\end{itemize}

\section{Related Work}
\label{sec:handling_categories}

\subsection{Optimal Quantification}

In \cite{linting2007nonparametric}, the authors give an account of another type of nonlinear PCA (NLPCA) based on optimal quantification (OQ), also known as optimal scaling \cite{de1987nonlinear}. In OQ, the goal is to convert each categorical variable to numerical by assigning each level within a categorical to a unique numerical value. Only then can linear PCA be applied to the transformed data to draw insights. Optimality is achieved through the use of an iterative algorithm for which the goal is to maximize the explained variance, similarly to our method. However, this method has major drawbacks that our proposed work addresses. That is, the NLPCA in \cite{linting2007nonparametric} is only non-linear in the way it quantifies categories and does not deal with nonlinear transformations of numerical variables. Furthermore, the nature of the transformations have to be specified manually (e.g., spline, piece-wise function). On the other hand, our proposed method not only transforms numerical variables on top of categorical (and ordinal) ones, but it does so with arbitrary complex transformations parametrized by neural networks and optimized using approximate gradient information. 

\subsection{Correspondence Analysis}

\label{sec:CA}

As mentioned in the Introduction, PCA is not well suited to deal with categorical variables. A popular and analogous method to PCA that deals with two categorical variables is correspondence analysis (CA) and its higher-dimensional variant multiple correspondence analysis (MCA) \cite{greenacre2006multiple}. Multiple correspondence analysis is applied to the contingency table of those categorical variables, where the rows and columns correspond to the levels (subcategories) of the categorical variables. The entries of the table are simply the total number of co-occurrences for the given levels (conditional frequencies).  Similarly to PCA, MCA finds orthogonal components that are used to project the data onto and for which the resulting projection maximizes the weighted variance, called inertia in CA. However, while MCA is a popular method, it has two major drawbacks that our proposed method aims to tackle. Firstly, the multitude of categories and levels within categories leads to an exponential explosion and to a contingency table with mostly low frequencies, skewing the results of the analysis, and rendering MCA unstable. Secondly, MCA does not natively handle continuous variables. To use MCA on a continuous variable, one would have to convert it to a categorical variable by binning the data. However, the binning process is arbitrary, leads to a loss of information, and does not allow to interpolate the data. 


\subsection{Multiple Factor Analysis}
\label{sec:MFA}

Multiple factor analysis (MFA) \cite{thurstone1931multiple} is a multivariate statistical method that acts on groups of variables and that can handle both groups of continuous and categorical variables simultaneously, in the same analysis. 
The aim of MFA is to study the similarity between individuals with respect to the set of variables, as well as the relationships between variables. Just like PCA and MCA, it outputs a simultaneous lower dimensional representation of the individuals and the variables (continuous or categorical) that can then be interpreted on the same plot. In particular, MFA can be seen as a mix of PCA (applied to continuous variables) and MCA (applied to categorical variables) such that the influence of each variable, or group of variables, is weighed in order not to dominate the component analysis. Practically, however, when performing MFA, each categorical variable has to be one-hot encoded which can lead to an explosion of the number of columns in the case of non-curated real-world datasets. On the other hand, our method allows each category to be represented by a single column in our analysis, where semantically close levels (subcategories) are mapped close to each other.

\section{Proposed Method for Non-Linear PCA}

Our method consists of three main components: (i) the parameterization of the non-linear transformations of the original variables, (ii) the objective function to optimize, and (iii) the optimization procedure used to maximize the proposed objective. Since PCA is at the backbone of our proposed objectives, we briefly introduce it, alongside the notation used throughout the paper, before diving into the 3 aforementioned components.

\subsection{Principal Components Analysis}

Let us denote the original data by $X \in \mathbb{R}^{n \times p}$, where $X^{(l)} \in \mathbb{R}^{n}$ represents the $j^{\textrm{th}}$ dimension (variable). We assume that the data is centered and standardized (that is, each column has zero mean and unit standard deviation). Finally, we denote the sample covariance matrix of $X$ as $S = \textrm{Cov}(X)$, where $S \in \mathbb{R}^{p \times p}$, and $S_{i,j} = \sigma_{i,j}$. PCA seeks an orthonormal basis $w_1, w_2, \cdots , w_p$ $\in \mathbb{R}^{p}$ to project the original data onto, so the variance $\textrm{Var}(Xw_j)$, $j=1,\ldots,p$ of the projection is maximized. Solving this optimization problem yields a matrix $W \in \mathbb{R}^{p \times p}$ whose rows $w_1, \ldots , w_p$ are to the eigenvectors of the covariance matrix $S$. The corresponding eigenvalues $\lambda_1, \ldots, \lambda_p$ represent the variance of each linear projection $Xw_j, j=1,\ldots,p$ , such that $\lambda_1 \leq \ldots \leq \lambda_p$. The data projected onto the principal components is denoted by $Z=XW$, where $Z \in \mathbb{R}^{n \times p}$.

\subsection{Non-linear Transformations}

In this paper, we consider transformations of the original variables where each original variable is mapped to a single transformed output. Formally, for each variable $X^{(l)}$, we define a function $\Phi_{l}: \mathbb{R} \to \mathbb{R}$, where $\Phi_{l}(X^{(l)}) = \widetilde{X}^{(l)}$, as depicted in Figure \ref{fig:transformations}. We choose to have a one-to-one mapping between the original variables and their transformations for two reasons: (i) interpretability of the solution, where each transformation can be referred back to its corresponding original variable, and (ii) we avoid having to deal with the trivial solution where all the transformations are colinear to each other, maximizing the first eigenvalue. Note that the special case where all the transformations have a constant output cannot happen as this would lead to a total variance of 0. Then, by concatenating the output of each transformation, we obtain the transformed dataset $\widetilde{X} \in \mathbb{R}^{n \times p}$.

\begin{figure}[ht!]
    \hspace{0.6cm}\includegraphics[width=0.8\linewidth]{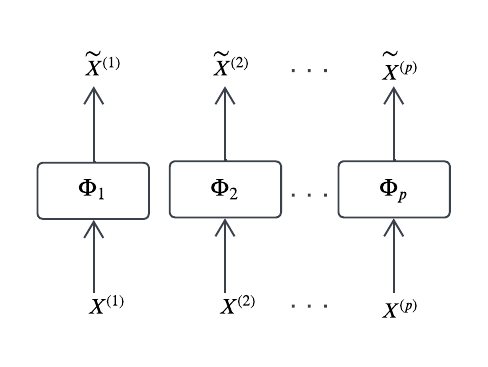}
    \vspace{-0.5cm}
    \caption{Diagram representation of the transformations of the original variables.}
    \label{fig:transformations}
\end{figure}

In this work, we choose to parameterize the transformations $\Phi_{l}$, for $l=1,\ldots,p$, as simple feed-forward neural networks to express nonlinear and arbitrary transformations. The architecture and parameters of the networks vary depending on the variable type that is being transformed (numerical vs categorical/nominal vs ordinal). The network architecture for each variable type is given in Appendix \ref{appendix:nn}.

\subsubsection{Handling of Categorical Variables}

For categorical variables, our method first one-hot encode each sub-level individually. Then, instead of passing a single variable (as would be the case for numerical variables) as input through the transformations $\Phi_{l}$, $l=1,\ldots,p$, we pass the one-hot encoded vector. The output, on the other hand, similarly to the numerical case, remains one-dimensional. Thus, each categorical variable contributes to a single column in the newly transformed space, the same as a numerical variable. This allows us to perform PCA directly on the newly transformed data without having to weigh variables in order to avoid a single variable dominating the component analysis. Recall that in MFA, each categorical variable is treated as a one-hot encoded vector leading to an explosion of the number of columns, which is not the case for our method. 

Finally, one can interpret the resulting scalar value for each sub-level: levels that are often associated (co-occurring) with the same numerical variables should be mapped close to each other. To represent a particular level within a category, we can simply compute the center of gravity (average) of the coordinates for all the individuals for that level. Therefore, similarly to CA and MCA, we can simultaneously represent the individuals and the levels on the same lower-dimensional plot.

\subsubsection{Notes on Overfitting}

It is possible for neural networks to overfit the data by arbitrarily mapping each point to maximize the variance when applying PCA. To this end, typical neural network regularization techniques (e.g., weight decay, dropout, limiting network size) can be readily applied. In addition, PCA is sensible to outliers and can thus be skewed by transformations of the original input that create artificial outliers or amplify already existing outliers. This effect could be alleviated, for instance, by using robust PCA \cite{candes2011robust} or by removing the outliers post-transformation. In addition, for datasets with a large number of variables, sparse PCA \cite{zou2006sparse} which yields a sparse representation of the loading vectors by applying an $\textit{l}_1$ penalty would be worth investigating.

\subsubsection{Notes on Interpretability}

Because one of the most popular features of PCA is its interpretability, we show that our method can be made interpretable by choosing an interpretable function class for the transformations $\Phi_{l}$, $l=1,\ldots,p$. Specifically, for the synthetic datasets in Section \ref{sec:synthetic_datasets}, in addition to neural networks, we also experiment using computational trees (shown in Figure \ref{fig:gp_diagram} in Appendix A). These trees are optimized using standard Genetic Programming (GP) \cite{koza1992genetic}, set a to a low depth to obtain simple and interpretable expressions. Another common way would be to fit an interpretable surrogate model, to explain a more complex model, in a post-hoc manner \cite{ribeiro2016should}.

\subsection{Proposed Objectives}
\label{sec:objective}

We propose two simple objectives to maximize, based on the amount of explained variance of the first $k$ principal components after having applied PCA to the transformed data $\widetilde{X}$ and obtaining the eigenvalues $\{\lambda_{j}\}_{j=1}^{p}$. Note that $k < p$ is a hyperparameter and should be set by the user depending on the problem.

\subsubsection{Maximizing the global Explained Variance}

The simplest objective is to maximize the sum of the variance explained by the first $k$ principal components. We denote this objective by $\mathcal{F}_{\textrm{global}}^{k}$ and define it as

\begin{equation}
    \label{eqn:total_variance_obj}
    \mathcal{F}_{\textrm{global}}^{k} = \sum_{j=1}^{k} \lambda_{j}.
\end{equation}

This objective makes it clear why each transformation $\Phi_{l}$ can only be a function of its corresponding original variable $X^{(l)}$. Indeed, had we allowed each transformation to be a function of all the original variables (i.e, $\Phi_{l} = f(X^{(1)}, \ldots, X^{(p)})$), then the trivial (and uninteresting) solution $\Phi_{l}	\propto \Phi_{m}, \forall \hspace{0.1cm} l, m$ would lead to the first principal component explaining 100\% of the variance contained in the transformed dataset. In addition, restricting $\Phi_{l}=f(X^{(l)})$ also has the advantage of being more interpretable, whereby each transformed variable is a function of its original self. This naturally allows us to use standard PCA analysis tools, such as biplots, to relate the transformed variables and data points to each other.

\subsubsection{Maximizing the Individual Variance Contributions}
\label{sec:variance_contrib}

Although the aforementioned objective is sensible, it assigns the same value to all the transformations shown in Figure \ref{fig:transformations}. However, it is conceivable that some transformations may be better than others, contributing to more variance. Therefore, we would like to have an objective that assigns a different value to each individual transformation. One possible solution is to single out the variance contribution towards the $k$ first eigenvalues of each transformed variable.

\begin{proposition}
\label{prop:1}
Let $c_{j,l}$ be the variance contribution of the $l^{\textrm{th}}$ variable $X^{(l)}$ towards the $j^{\textrm{th}}$ eigenvalue $\lambda_j$. Then, we have that 

\begin{equation}
\label{eqn}
c_{j,l} = w^2_{j,l}\sigma_{l,l} + \sum_{i=1, i \neq l}^{p}w_{j,i}w_{j,l}\sigma_{i,l}.
\end{equation}
\end{proposition}
\begin{proof} 
The result directly follows from decomposing the eigenvalue:

\begin{equation} \label{eq:eigenvalue}
\begin{split}
\lambda_j & = \textrm{Cov}(w_jX) \\
 & = \sum_{i=1}^{p}\sum_{l=1}^{p}w_{j,l}w_{j,i}\sigma_{i,l} \\
 & = \sum_{i=1}^{p}w^{2}_{j,i}\sigma_{i,i} + 2 \sum_{i=1}^{p-1} \sum_{l=i+1}^{p} w_{j,i}w_{j,l}\sigma_{i,l}. 
\end{split}
\end{equation}

The first term corresponds to the variance, while the second term corresponds to the $p \choose 2$ covariance pairs, where the factor of 2 comes from the fact that $S$ is symmetric. Then, by extracting the contribution of a single variable $l$ from Equation \ref{eqn:total_variance_obj}, we obtain the required result.
\end{proof}

The result of Proposition \ref{prop:1} can then be used to calculate the contributions $\{c_{j,l}\}_{j,l=1}^{p}$ of each variable for all components. This formulation allows us to optimize the variance contribution of each transformation $\Phi_{l}$ individually:

\begin{equation}
\label{eqn:contrib_objective}
    \mathcal{F}_{l}^{k} = \sum_{j=1}^{k} c_{j,l},
\end{equation}

as opposed to only having access to the overall explained variance $\mathcal{F}_{\textrm{global}}^{k}$. This formulation is our new proposed objective. We show in Proposition \ref{prop:2} that it is a fine-grained version of the total explained variance in Equation \ref{eqn:total_variance_obj}. Throughout the rest of this paper, we refer to this formulation as the \textit{partial} objective. Furthermore, we refer to the regular formulation in Equation \ref{eqn:total_variance_obj} as the \textit{partial} objective.

\begin{proposition}
\label{prop:2}
The formulation in Equation \ref{eqn:contrib_objective} is a more granular version of the global variance objective in Equation \ref{eqn:total_variance_obj} and we have that
\begin{equation}
    \label{eqn:equivalence}
    \mathcal{F}_{\textrm{global}}^{k} = \sum_{l=1}^{p} \mathcal{F}_{l}^{k}.
\end{equation}
\end{proposition}
\begin{proof} 
Substituting Equation \ref{eqn:contrib_objective} into the RHS of Equation \ref{eqn:equivalence}, we get 
\begin{equation}
\nonumber
    \sum_{l=1}^{p} \mathcal{F}_{l}^{k} = \sum_{l=1}^{p}\sum_{j=1}^{k} c_{j,l} = \sum_{j=1}^{k}\sum_{l=1}^{p} c_{j,l} = \sum_{j=1}^{k} \lambda_{j},
\end{equation}
where the last equality comes from the definition of $c_{j,l}$.
\end{proof}

Intuitively, this formulation should provide a stronger learning signal since it is able to inform the optimization for each transformation individually, rather than when using the global variance objective.

\subsubsection{Scaling the Global Explained Variance}
\label{sec:scaling_total_variance}

The \textit{global} objective in Equation \ref{eqn:total_variance_obj} corresponds to the total variance explained by all the variables. On the hand, the \textit{partial} objective is granular: each variable is responsible for only a fraction of the total explained variance. Proposition \ref{prop:2} implies that on average (assuming each variable contributes to the same amount of variance) we have that $\mathcal{F}_{\textrm{global}}^{k} = p\mathcal{F}^{k}$, where $\mathcal{F}^{k}$ is the variance contribution of a single variable, for the $k^{th}$ eigenvalue. This directly affects the magnitude (by a factor of $p$) of the learning step in evolution strategies. Therefore, to make sure that the step size is of same magnitude for both the \textit{partial} and \textit{global} objectives, we scale down $\mathcal{F}_{\textrm{global}}^{k}$ by a factor of $p$, as shown in Algorithm 1. This control ensures that any observed advantages are a result of the \textit{partial} objective's strong and robust signal, rather than an artifact of differing optimization scales.

\subsection{Optimization using Evolution Strategies}
\label{sec:optimization}

Since the eigendecomposition of the covariance matrix is not a differentiable computation, the objectives proposed in \ref{sec:objective} are thus non-differentiable and we cannot compute their derivatives. Furthermore, our network does not output a distribution over actions to be sampled from and evaluated, rendering standard reinforcement learning (RL) algorithms such as Policy Gradients \cite{sutton1999policy} not suitable. To address this problem, we adopt a class of derivative-free black-box optimization algorithms, called Evolution Strategies (ES) \cite{rechenberg1978evolutionsstrategien}. At the core, ES are iterative population-based methods used to update the parameters of a solution (a neural network in our case) by perturbing its parameters, computing the objective value for each perturbation, updating the parameters through gradient approximation, and iterating the process until some stopping criterion is met. 
In a recent work \cite{salimans2017evolution}, the authors adapted a version of ES, called Natural Evolution Strategies (NES) \cite{wierstra2014natural, sun2009efficient} to successfully optimize the weights of relatively large neural networks to play Atari games. In essence, ES is akin to approximating the first-order derivative by using randomized (for the chosen direction) finite differences in high dimension. Although approximating the gradient in high dimension using finite differences may seem unfeasible due to the curse of dimensionality, the authors argue that what matters is the intrinsic dimension of the optimization problem and not the number of dimensions of a potentially overparameterized neural network \cite{salimans2017evolution}.

Formally, let us denote the parameters of our neural network by $\theta$, where we assume that $\theta \sim p_{\mu}(\theta)=N(\mu, I\sigma^2)$, for $\sigma^2$ fixed. The goal is to maximize the average objective value $\mathbb{E}_{\theta \sim p_{\mu}}\mathcal{F}(\theta)$ over the whole population by optimizing for $\mu$ with stochastic gradient ascent. Note that we do not have to update $\sigma$ as it is fixed. Therefore, for each stochastic update, we have to compute $\nabla_{\mu} \mathbb{E}_{\theta \sim p_{\mu}}\mathcal{F}(\theta)$. This gradient is in fact the score function estimator, also known as the REINFORCE estimator \cite{williams1992simple} and can be written as 

\begin{equation}
    \nabla_{\mu} \mathbb{E}_{\theta \sim p_{\mu}}\mathcal{F}(\theta) = \mathbb{E}_{\theta \sim p_{\mu}}[\mathcal{F}(\theta)\nabla_{\mu}\textrm{log}(p_{\mu}(\theta))].
\end{equation}

Now, since $\nabla_{\mu}\textrm{log}(p_{\mu}(\theta))=\frac{(\theta-\mu)}{\sigma^2}$. and $\theta = \mu + \epsilon \sigma$, $\epsilon \sim N(0, I)$, we can write $\nabla_{\mu}\textrm{log}(p_{\mu}(\theta)) = \frac{\epsilon}{\sigma}$. We also have that $\mathbb{E}_{\theta \sim p_{\mu}}\mathcal{F}(\theta) = \mathbb{E}_{\epsilon \sim N(0, I)}\mathcal{F}(\theta + \epsilon \sigma)$.  Putting all these pieces together, we obtain

\begin{equation}
\nabla_{\mu} \mathbb{E}_{\theta \sim p_{\mu}}\mathcal{F}(\theta) = \frac{1}{\sigma} \mathbb{E}_{\epsilon \sim N(0, I)}[\mathcal{F}(\theta + \epsilon \sigma)\epsilon],
\end{equation}

which can be estimated using Monte Carlo sampling \cite{sehnke2010parameter}. The general algorithm to implement the Monte Carlo sampling is described in Algorithm \ref{alg:evolution-strategies}. At each step $t$, we perturb the network parameters by injecting normal noise, $P$ independent times, and compute the objective for each perturbation. Then, we update the parameters by taking a step in the direction of the average weighted (by the objective value) noise samples. 

Note that as a shorthand, in Algorithm \ref{alg:evolution-strategies}, we use $\theta$ to denote the set of all parameters $\{\theta^{l}\}_{l=1}^{p}$ corresponding to the set of neural networks $\{\Phi_{l}\}_{l=1}^{p}$. In the case where $\mathcal{F}_{\textrm{global}}^{k}$ (see Equation \ref{eqn:total_variance_obj}) is used as an objective function, all the network parameters can be updated at the same time since they share the same objective function (and thus have the same gradient). This can be seen in line 6 of Algorithm \ref{alg:evolution-strategies}. On the other hand, for $\mathcal{F}_{l}^{k}$ (see Equation \ref{eqn:contrib_objective}), we have to update each network parameters $\theta^{l}$ separately using a different gradient approximation for each network. A more detailed description of the training procedure using ES as well as the values of the ES parameters are given in Appendix \ref{appendix:es}.

\begin{algorithm}[tb]
\caption{Evolution Strategies for non-linear PCA}
\label{alg:evolution-strategies}

\begin{algorithmic}[1]

\STATE \textbf{Input:} number of generations: $T$, learning rate: $\alpha$, standard deviation: $\sigma$, population size: $P$, initial networks parameters: $\{\theta_{0}^{l}\}_{l=1}^{p}$, number of principal components: $k$, partial objective: $is\_partial$.

\FOR{$t = 0, \dots, T-1$}
\STATE Sample $\epsilon_1, \ldots, \epsilon_P \sim N(0, I).$
\IF{not $is\_partial$}
\STATE $\mathcal{F}_{\textrm{global},i}^{k}=\mathcal{F}_{\textrm{global}}^k(\theta_t + \sigma \epsilon_i)$, \hspace{0.1cm} for $i=1,\ldots,P$.
\STATE $\theta_{t+1} \leftarrow \theta_{t} + \frac{\alpha}{pP\sigma} \sum_{i=1}^{P}\mathcal{F}_{\textrm{global},i}^{k}\epsilon_i$
\ELSE
\FOR{$l = 1, \dots, p$}
\STATE $\mathcal{F}_{l, i}^{k}=\mathcal{F}_{l}^k(\theta_t + \sigma \epsilon_i)$, \hspace{0.1cm} for $i=1,\ldots,P$.
\STATE $\theta_{t+1}^{l} \leftarrow \theta_{t} + \frac{\alpha}{P\sigma} \sum_{i=1}^{P}\mathcal{F}_{l,i}^{k}\epsilon_i$
\ENDFOR
\ENDIF
\ENDFOR
\STATE \textbf{Output:} optimized networks parameters: $\{\theta_{T}^{l}\}_{l=1}^{p}$.

\end{algorithmic}
\end{algorithm}

\subsection{Computational Complexity}

We analyze the computational complexity of the proposed algorithm. For the cost of the objective function, we focus on the dominant cost of the Singular Value Decomposition (SVD). We let $C_{\text{SVD}} = O(\min(n^{2}p,\, np^{2}))$ denote the computational cost of performing a single SVD operation on a $n \times p$ matrix. The overall complexity of the algorithm depends on objective function being used (global vs. partial) and the frequency of the PCA updates.

\subsubsection{Sequential Complexity}

\paragraph{Global objective function (\texttt{not is\_partial})}
In this setting, the algorithm iterates through $T$ generations. In each epoch $t$, it generates a population of size $P$. For each candidate $i \in \{1, \dots, P\}$, the objective $\mathcal{F}_{\text{global},i}^{k}$ is computed, which requires performing PCA.
The total computational complexity is:
\begin{equation}
    \mathcal{T}_{\text{global}} \approx O(T \cdot P \cdot C_{\text{SVD}}).
\end{equation}

\paragraph{Partial objective function (\texttt{is\_partial})}
In this setting, the algorithm includes an additional nested loop over $p$ partial components. The rest remains the same, and so, we have that the complexity scales linearly with $p$:
\begin{equation}
    \mathcal{T}_{\text{partial}} \approx O(T \cdot p \cdot P \cdot C_{\text{SVD}})
\end{equation}

Now, if the PCA basis is updated only every $m$ generations (instead of at every iteration/perturbation), the number of SVD operations is reduced by a factor of $m$. In the in-between steps, the objective is computed using the cached basis from the most recent SVD. We would have:

\begin{equation}
    \mathcal{T}_{\text{global}} \approx O\left( \frac{T}{m} \cdot P \cdot C_{\text{SVD}} \right).
\end{equation}

\subsubsection{Parallel Complexity}

Evolution Strategies are naturally parallelizable. The population loop (evaluating $i=1, \dots, P$) consists of independent operations that can be distributed across $W$ workers. Let $W$ be the number of parallel workers. The parallel complexity of the algorithm becomes

\begin{equation}
    \mathcal{T}_{\text{parallel}} \approx O\left( T \cdot \left\lceil \frac{P}{W} \right\rceil \cdot C_{\text{SVD}} \right).
\end{equation}

If the event that $W = P$ the complexity per epoch reduces to $O(T \cdot C_{\text{SVD}})$. Note that the inner loop for the \textit{partial} objective is also parallelizable and would thus reduce to the above too.

\section{Datasets}
\label{sec:datasets}

\subsection{Synthetic Datasets}
\label{sec:synthetic_datasets}

As a basis for testing our method, we generate three synthetic datasets (see Appendix \ref{appendix:synthetic_data}): (i) \textit{nested circles}, (ii) \textit{nested spheres}, and (iii), \textit{alternate stripes}. These datasets are devised in such a way that each dimension contributes to an orthogonal source of variance. Therefore, performing linear PCA on the original variables leads to an equal amount of explained variance in each principal component, as shown in Table \ref{tab:results}. On the other hand, there exist transformations that would lead to a high amount of explained variance in the first principal components, as explained in the Appendix.

\subsection{Real-world Data}
\label{exp:real-world-data}


In this section, we give an overview of the five datasets used in the experiments. The datasets are sourced from the machine learning dataset repository $\textit{openml.org}$\footnote{\url{https://openml.org/}}. The filtering used to select the datasets on $\textit{openml.org}$ is the following: (i) \textit{Status}:  \textit{Verified}, (ii) \textit{Instances}:  \textit{1000s}, (iii) \textit{Features}:  \textit{10s}, (iv) \textit{Target}:  \textit{Binary classification or multi-class}. We then select 5 datasets from a wide range of domains (e.g., credit check, meteorological, medical) and attribute types (continuous, boolean, categorical, and a mix of those). The characteristics of these datasets are described in Table \ref{table:datasets} below.

\begin{table}[h!]
\centering
\caption{Description of the real-world datasets used in the experiments.}
\label{table:datasets}
\begin{tabular}{|c|c|c|c|}
\textbf{Datasets} & \textbf{\# instances} & \textbf{\# features} & \textbf{\# classes} \\
\hline
credit-g & 1000 & 21 (14 cat.) & 2 \\
ionosphere & 351 & 35 (1 cat.) & 2 \\
breast-cancer & 286 & 10 (10 cat.) & 2 \\
heart-statlog & 270 & 14 (7 cat.) & 2 \\
wine-dataset & 178 & 13 (1 cat.) & 2  \\
\end{tabular}
\end{table}

\section{Experiments}

\subsection{Training Protocol}

Our experimental setting and training protocol are carefully designed to ensure that there is no data leakage and that the evaluation is robust. In Figure \ref{fig:training_protocol_diagram} below we show a full experimental run. For robustness, we repeat this process 15 times for each dataset with different (random) data splits and weight initialization. To avoid data leakage, we randomly split the data into training and validation set, using a $0.75/0.25$ split. The neural network weights are then optimized using ES (Algorithm 1) on the training set. After each generation, the optimized networks are used to both transform the training and validation data, as depicted in Figure \ref{fig:transformations}. Finally, we perform PCA on the transformed data and report the eigenvalues as in Table \ref{tab:results} and Figure \ref{fig:training_curves}.

\FloatBarrier

\begin{figure}[htb!]
    \centering
    \includegraphics[width=1.1\linewidth]{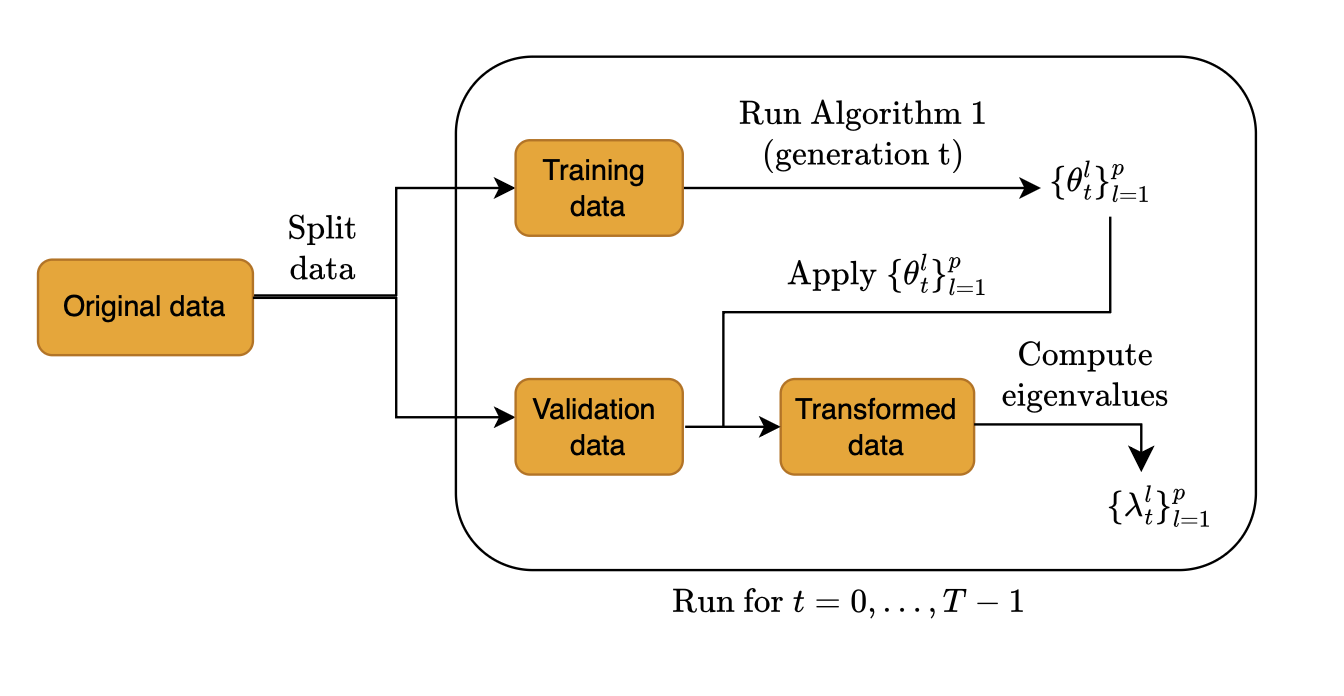}
    \caption{Training and evaluation protocol.}
    \label{fig:training_protocol_diagram}
\end{figure}

\subsection{How to Evaluate our Method?}
\label{sec:evaluation_metric}

There is no standard way to evaluate the quality of of a resulting lower-dimensional representation. The appropriate evaluation metric largely depends on the application and on the end-goal of the practitioner when performing dimensionality reduction. For instance, in \cite{uriot2022genetic}, the authors use both the reconstruction error and the predictive performance of the lower-dimensionality representation in supervised tasks. However, it is reported that both metrics come with their own drawbacks such as putting too much importance on preserving colinear features. 

Therefore, in this paper, in order to keep the comparison as straightforward and fair as possible, we only compare our method to other component-based methods for which the goal is also to find a representation that maximizes the explained variance in the transformed space. This allows us to compare the obtained eigenvalues for the various methods across the datasets.
As mentioned in the introduction, kPCA is the only true non-linear extension of PCA from the NLPCA research field. We can compute and directly compare the explained variance of the few leading principals components obtained in kPCA with those obtained for our method. In addition, as a baseline, we will compute the eigenvalues obtained by applying linear PCA to the original data.

Qualitative assessments such as graphical representations (e.g., biplots) are also important in order to assess the lower-dimensional representation and get an understanding of what lower dimensions represent.

\section{Results}
\label{sec:results}

\subsection{Explained Variance}

In Table \ref{tab:results} below, we can see that our novel objective (Equation \ref{eqn:contrib_objective}), denoted as \textbf{ES-Partial}, is consistently explaining either the most variance, or the second most. In particular, for $k=2$ on the real-world datasets, it explains significantly more variance than the other methods. For $k=1$, both \textbf{ES-Partial} and \textbf{kPCA} perform well. Note that the \textbf{kPCA} comparison is tilted in its favor since we report the best result out of 4 kernels (rbf, cosine, polynomial, sigmoid). In either cases, we do not fine-tune the neural networks or the kernels hyperparameters. We use the default parameters from the Scikit-learn's implementation of KernelPCA \cite{pedregosa2011scikit}. Furthermore, the explained variance is still increasing with the number of generations (see Figure \ref{fig:training_curves}) on several datasets, suggesting that more gains are achievable. As expected, performing PCA on the transformed original data using evolution strategies always explains more variance than linear PCA on the original variables (\textbf{PCA} vs. \textbf{ES-Global} and \textbf{ES-Partial}). Note that for \textbf{PCA}, we can have that the first principal component explains less than $\frac{1}{p}$ of the total variance. This is because we report the results on the validation set while the PCA parameters are fitted on the training set.

\begin{table*}[t]
\caption{Mean explained variance (over 15 independent runs with different training and validation sets), reported on validation data. Results are shown for $k=1$ and $k=2$. The methods are: linear PCA, kernel PCA (best kernel shown in parentheses), and ES-PCA using the \textit{global} and \textit{partial} objectives. Best results per dataset and per $k$ are shown in bold, second best are underlined.}
\vspace{0.15cm}
\label{tab:results}
\centering
\begin{tabular}{l c l c c c l c c}
\toprule
& \multicolumn{4}{c}{\textbf{$k=1$}} & \multicolumn{4}{c}{\textbf{$k=2$}} \\
\cmidrule(lr){2-5} \cmidrule(lr){6-9}
\textbf{Dataset}
& \textbf{PCA}
& \textbf{kPCA}
& \textbf{ES-Global}
& \textbf{ES-Partial}
& \textbf{PCA}
& \textbf{kPCA}
& \textbf{ES-Global}
& \textbf{ES-Partial} \\
\midrule
Alternate stripes
& 0.513
& 0.573 (rbf)
& \textbf{0.715}
& \underline{0.693}
& N/A
& N/A
& N/A
& N/A \\

Circles
& 0.498
& \textbf{0.785} (rbf)
& 0.714
& \underline{0.736}
& N/A
& N/A
& N/A
& N/A \\

Spheres
& 0.341
& \underline{0.576} (rbf)
& 0.517
& \textbf{0.596}
& N/A
& N/A
& N/A
& N/A \\

Wine
& \underline{0.370}
& \underline{0.370} (linear)
& 0.364
& \textbf{0.448}
& \underline{0.607}
& \textbf{0.620} (cos)
& 0.544
& \textbf{0.620} \\

Heart-statlog
& 0.230
& \textbf{0.295} (poly)
& 0.241
& \underline{0.268}
& 0.329
& \underline{0.361} (poly)
& 0.345
& \textbf{0.381} \\

Ionosphere
& 0.256
& \textbf{0.377} (cos)
& 0.238
& \underline{0.325}
& 0.381
& \underline{0.505} (cos)
& 0.399
& \textbf{0.543} \\

Breast cancer
& 0.104
& 0.127 (poly)
& \underline{0.231}
& \textbf{0.245}
& 0.176
& 0.246 (poly)
& \underline{0.375}
& \textbf{0.418} \\

German credit
& 0.067
& 0.070 (cos)
& \underline{0.119}
& \textbf{0.153}
& 0.123
& 0.132 (cos)
& \underline{0.210}
& \textbf{0.253} \\
\bottomrule
\end{tabular}
\end{table*}

Our novel objective function significantly outperforms the \textit{global} approach across nearly all datasets, with the exceptions of \textit{circles} and \textit{alternate stripes}. Specifically, the performance gains of \textbf{ES-Partial} over \textbf{ES-Global} appear to scale with the dimensionality of the dataset. To quantify this, let $\lambda_{\text{partial}}$ and $\lambda_{\text{global}}$ represent the explained variance achieved by each respective method. We define the relative difference as:
\begin{equation}
    D = 100 \times \frac{\lambda_{\text{partial}} - \lambda_{\text{global}}}{\lambda_{\text{partial}}}
\end{equation}
Across the eight datasets for $k=1$ and five for $k=2$, this difference $D$ correlates strongly with dataset dimensionality, yielding correlation coefficients of $\rho = 0.87$ and $\rho = 0.98$, respectively.

\subsection{Qualitative Results}

We have shown that our method outperforms linear PCA across all benchmark datasets, as well as broadly outperforming kPCA (best result reported out of four kernels: rbf, cosine, quadratic, sigmoid). In this section, we show on the synthetic datasets that our method yields principal components that are intuitively interpretable. First, we show this on the results produced by applying Genetic Programming on the synthetic datasets. Then, we show a side-by-side comparisons of biplots for linear PCA and for ES-PCA. 

In Figure \ref{fig:gp_viz}, we show the transformations found by GP that maximize the explained variance (Equation \ref{eqn:contrib_objective}), for $k=1$. For the \textit{nested spheres} dataset, we find transformations for which the circular structure is lost. The first principal component amounts to the difference in radius between the inner and outer spheres. For the \textit{alternate stripes} dataset, we have that $\lambda_1 = 1$, as we are able to discover the periodicity of ${X}^{(1)}$ and the relationship between ${X}^{(1)}$ and ${X}^{(2)}$ through $\epsilon$. This is because we have that $\textrm{cos}(n\pi + \epsilon)) = \textrm{cos}(\epsilon)$.


\begin{figure}[ht!]
    \centering
    \begin{subfigure}{0.49\columnwidth}
        \centering
        \includegraphics[width=\linewidth]{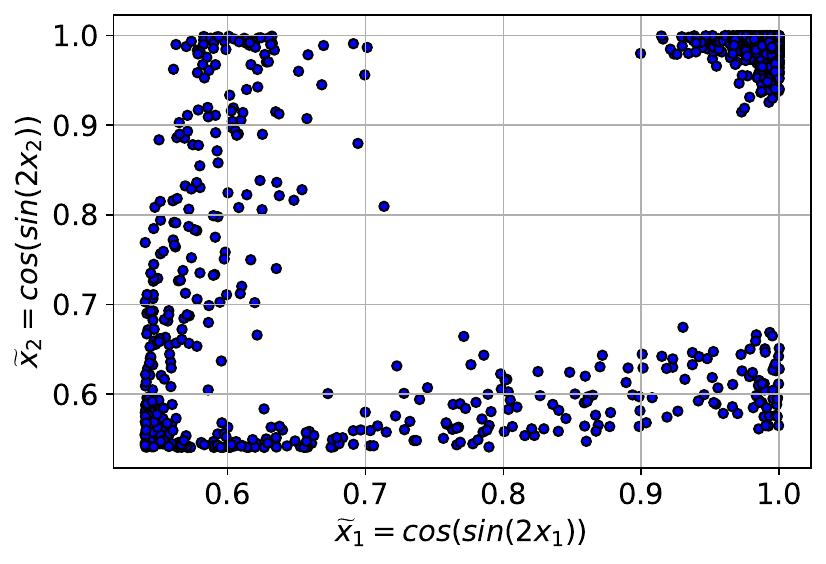}
        \label{fig:circles:transformed}
    \end{subfigure}\hfill
    \begin{subfigure}{0.50\columnwidth}
        \centering
        \includegraphics[width=\linewidth]{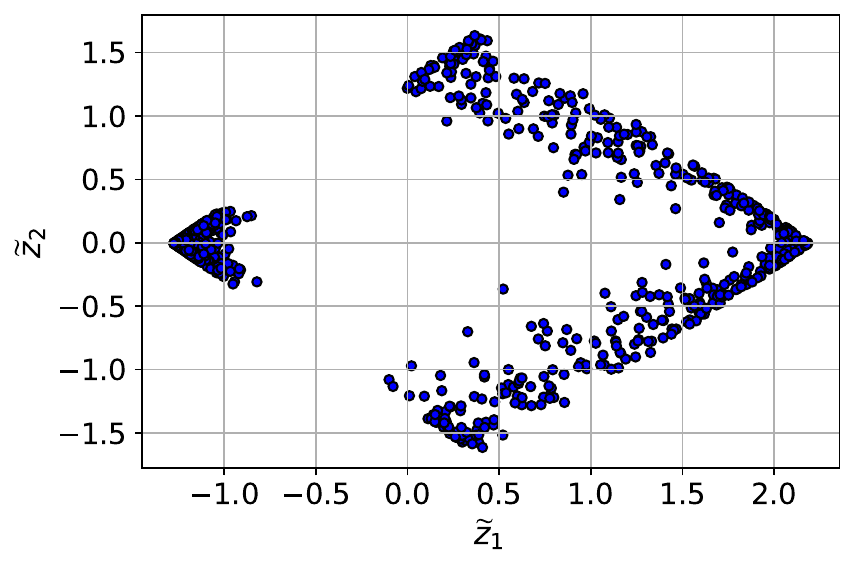}
        \label{fig:circles:pc}
    \end{subfigure}
    \begin{subfigure}{0.49\columnwidth}
        \centering
        \includegraphics[width=\linewidth]{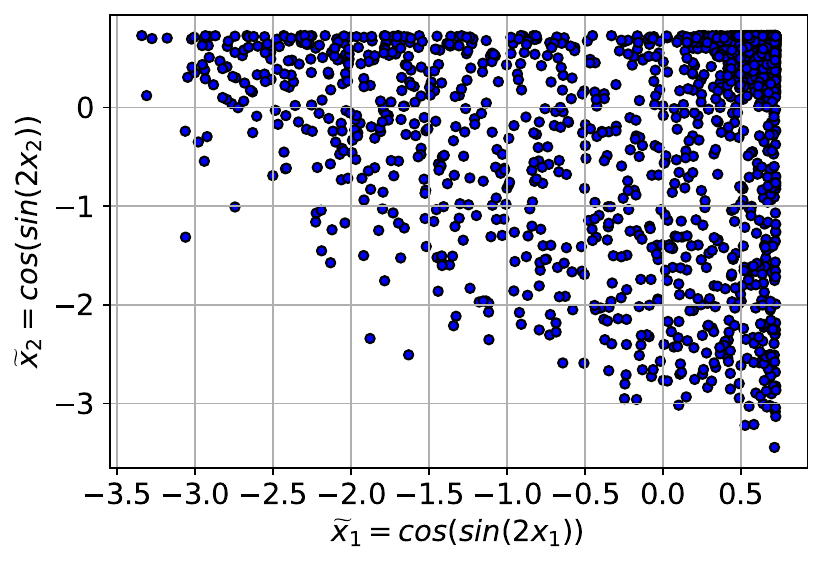}
        \label{fig:spheres:transformed}
    \end{subfigure}\hfill
    \begin{subfigure}{0.50\columnwidth}
        \centering
        \includegraphics[width=\linewidth]{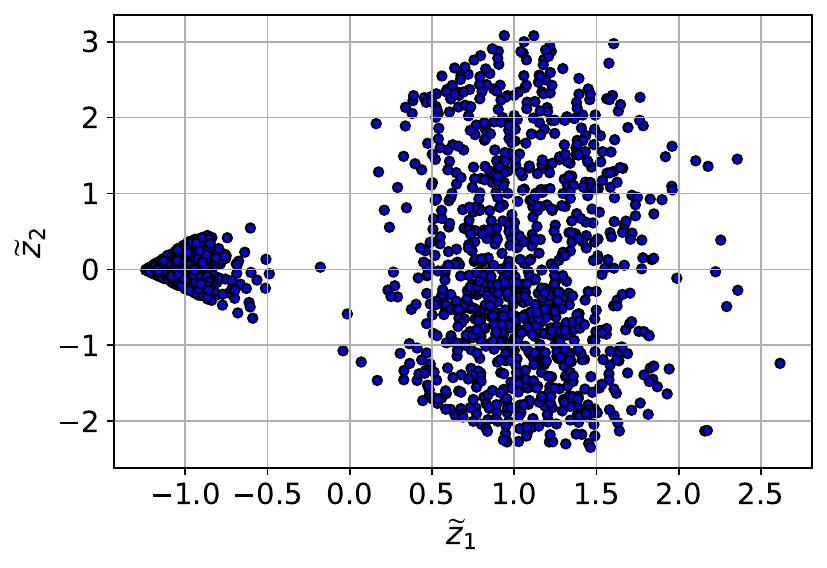}
        \label{fig:spheres:pc}
    \end{subfigure}
    %
    \caption{
        \textit{Top}: nested circles.
        \textit{Bottom}: nested spheres.
    }
    \label{fig:gp_viz}
\end{figure}

\section{Conclusion}
In this paper, we introduced a novel framework for non-linear Principal Component Analysis (NLPCA) that bridges the gap between the flexibility of deep learning and the interpretability of classical multivariate statistics. By parameterizing variable transformations through neural networks and optimizing them via Evolution Strategies (ES), our method bypasses the non-differentiability constraints of standard eigendecomposition. 

Our primary contribution is the introduction of a \textit{partial} objective function. Unlike global variance maximization, which provides a single scalar signal for the entire system, our partial objective decomposes the global explained variance to isolate and maximize the specific contribution of each variable individually. This provides a significantly stronger and more robust learning signal, particularly in higher dimensions. We show that our novel objective significantly outperforms as the dimensionality of the datasets grows.

Furthermore, our framework offers a unified approach to data handling. By mapping each original variable to a single transformed output, we natively process categorical, ordinal, and numerical variables within the same analysis. This avoids the dimensional explosion typically caused by one-hot encoding in methods like Multiple Factor Analysis (MFA), as each category is represented by only a single column in the transformed space. Consequently, our approach preserves a one-to-one mapping that allows for the continued use of standard interpretability tools, such as biplots, to visualize non-linear relationships.

\section{Future Work}
Building upon this foundation, several avenues for further research remain. A primary area of interest is the exploration of alternative neural architectures. Given the relatively small size of the networks, they may be well-suited for Kolmogorov-Arnold Networks (KANs) \cite{liu2024kan}, which could offer greater parameter efficiency and interpretability. Additionally, while we have demonstrated the effectiveness of vanilla Evolution Strategies, investigating more computationally efficient gradient-approximation methods. For instance, Natural Evolution Strategies (NES) \cite{wierstra2014natural} or Covariance Matrix Adaptation Evolution Strategy (CMA-ES) \cite{hansen2001completely} could improve convergence rates. For applications requiring symbolic clarity, modern interpretable versions of Genetic Programming, such as GOMEA \cite{virgolin2017scalable}, could be used to parameterize the variable transformations.

From a computational efficiency standpoint, we suggest investigating the impact of computing the PCA eigendecomposition only every $n$ iterations rather than at every generation. As noted in our complexity analysis, the eigendecomposition is a major bottleneck and reducing its frequency could significantly accelerate training without substantially degrading the quality of the learned transformations. 

Finally, future work should involve the integration of specialized regularization techniques, such as Sparse PCA or Robust PCA, to handle high-dimensional noise and seek more parsimonious representations. We also aim to expand our empirical validation to a wider array of datasets, focusing on qualitative comparisons using biplots to better understand the non-linear relationships captured between disparate categorical and numerical variables.

\bibliographystyle{ACM-Reference-Format}
\bibliography{acmart}

@article{pearson1901liii,
  title={LIII. On lines and planes of closest fit to systems of points in space},
  author={Pearson, Karl},
  journal={The London, Edinburgh, and Dublin philosophical magazine and journal of science},
  volume={2},
  number={11},
  pages={559--572},
  year={1901},
  publisher={Taylor \& Francis}
}

@article{scholkopf1998nonlinear,
  title={Nonlinear component analysis as a kernel eigenvalue problem},
  author={Sch{\"o}lkopf, Bernhard and Smola, Alexander and M{\"u}ller, Klaus-Robert},
  journal={Neural computation},
  volume={10},
  number={5},
  pages={1299--1319},
  year={1998},
  publisher={MIT Press}
}

@inproceedings{uriot2022genetic,
  title={On genetic programming representations and fitness functions for interpretable dimensionality reduction},
  author={Uriot, Thomas and Virgolin, Marco and Alderliesten, Tanja and Bosman, Peter AN},
  booktitle={Proceedings of the Genetic and Evolutionary Computation Conference},
  pages={458--466},
  year={2022}
}

@book{koza1992genetic,
  title={Genetic programming: {O}n the programming of computers by means of natural selection},
  author={Koza, John R},
  volume={1},
  year={1992},
  publisher={MIT Press}
}

@article{salimans2017evolution,
  title={Evolution strategies as a scalable alternative to reinforcement learning},
  author={Salimans, Tim and Ho, Jonathan and Chen, Xi and Sidor, Szymon and Sutskever, Ilya},
  journal={arXiv preprint arXiv:1703.03864},
  year={2017}
}

@article{scikit-learn,
 title={Scikit-learn: Machine Learning in {P}ython},
 author={Pedregosa, F. and Varoquaux, G. and Gramfort, A. and Michel, V.
         and Thirion, B. and Grisel, O. and Blondel, M. and Prettenhofer, P.
         and Weiss, R. and Dubourg, V. and Vanderplas, J. and Passos, A. and
         Cournapeau, D. and Brucher, M. and Perrot, M. and Duchesnay, E.},
 journal={Journal of Machine Learning Research},
 volume={12},
 pages={2825--2830},
 year={2011}
}

@article{gonzalez2010measurement,
  title={Measurement of areas on a sphere using Fibonacci and latitude--longitude lattices},
  author={Gonz{\'a}lez, {\'A}lvaro},
  journal={Mathematical Geosciences},
  volume={42},
  number={1},
  pages={49--64},
  year={2010},
  publisher={Springer}
}

@incollection{rechenberg1978evolutionsstrategien,
  title={Evolutionsstrategien},
  author={Rechenberg, Ingo},
  booktitle={Simulationsmethoden in der Medizin und Biologie},
  pages={83--114},
  year={1978},
  publisher={Springer}
}

@article{sutton1999policy,
  title={Policy gradient methods for reinforcement learning with function approximation},
  author={Sutton, Richard S and McAllester, David and Singh, Satinder and Mansour, Yishay},
  journal={Advances in neural information processing systems},
  volume={12},
  year={1999}
}

@article{hansen2001completely,
  title={Completely derandomized self-adaptation in evolution strategies},
  author={Hansen, Nikolaus and Ostermeier, Andreas},
  journal={Evolutionary computation},
  volume={9},
  number={2},
  pages={159--195},
  year={2001},
  publisher={MIT Press}
}

@article{wierstra2014natural,
  title={Natural evolution strategies},
  author={Wierstra, Daan and Schaul, Tom and Glasmachers, Tobias and Sun, Yi and Peters, Jan and Schmidhuber, J{\"u}rgen},
  journal={The Journal of Machine Learning Research},
  volume={15},
  number={1},
  pages={949--980},
  year={2014},
  publisher={JMLR. org}
}

@inproceedings{sun2009efficient,
  title={Efficient natural evolution strategies},
  author={Sun, Yi and Wierstra, Daan and Schaul, Tom and Schmidhuber, J{\"u}rgen},
  booktitle={Proceedings of the 11th Annual conference on Genetic and evolutionary computation},
  pages={539--546},
  year={2009}
}

@article{williams1992simple,
  title={Simple statistical gradient-following algorithms for connectionist reinforcement learning},
  author={Williams, Ronald J},
  journal={Machine learning},
  volume={8},
  number={3},
  pages={229--256},
  year={1992},
  publisher={Springer}
}

@article{sehnke2010parameter,
  title={Parameter-exploring policy gradients},
  author={Sehnke, Frank and Osendorfer, Christian and R{\"u}ckstie{\ss}, Thomas and Graves, Alex and Peters, Jan and Schmidhuber, J{\"u}rgen},
  journal={Neural Networks},
  volume={23},
  number={4},
  pages={551--559},
  year={2010},
  publisher={Elsevier}
}

@article{cerda2018similarity,
  title={Similarity encoding for learning with dirty categorical variables},
  author={Cerda, Patricio and Varoquaux, Ga{\"e}l and K{\'e}gl, Bal{\'a}zs},
  journal={Machine Learning},
  volume={107},
  number={8},
  pages={1477--1494},
  year={2018},
  publisher={Springer}
}

@inproceedings{hirschfeld1935connection,
  title={A connection between correlation and contingency},
  author={Hirschfeld, Hermann O},
  booktitle={Mathematical Proceedings of the Cambridge Philosophical Society},
  volume={31},
  pages={520--524},
  year={1935},
  organization={Cambridge University Press}
}

@book{greenacre2006multiple,
  title={Multiple correspondence analysis and related methods},
  author={Greenacre, Michael and Blasius, Jorg},
  year={2006},
  publisher={Chapman and Hall/CRC}
}

@article{thurstone1931multiple,
  title={Multiple factor analysis.},
  author={Thurstone, Louis Leon},
  journal={Psychological review},
  volume={38},
  number={5},
  pages={406},
  year={1931},
  publisher={Psychological Review Company}
}

@article{kruger2008developments,
  title={Developments and applications of nonlinear principal component analysis--a review},
  author={Kruger, Uwe and Zhang, Junping and Xie, Lei},
  journal={Principal manifolds for data visualization and dimension reduction},
  pages={1--43},
  year={2008},
  publisher={Springer}
}

@article{hastie1989principal,
  title={Principal curves},
  author={Hastie, Trevor and Stuetzle, Werner},
  journal={Journal of the American Statistical Association},
  volume={84},
  number={406},
  pages={502--516},
  year={1989},
  publisher={Taylor \& Francis}
}

@article{kramer1991nonlinear,
  title={Nonlinear principal component analysis using autoassociative neural networks},
  author={Kramer, Mark A},
  journal={AIChE journal},
  volume={37},
  number={2},
  pages={233--243},
  year={1991},
  publisher={Wiley Online Library}
}

@article{tan1995reducing,
  title={Reducing data dimensionality through optimizing neural network inputs},
  author={Tan, Shufeng and Mayrovouniotis, Michael L},
  journal={AIChE Journal},
  volume={41},
  number={6},
  pages={1471--1480},
  year={1995},
  publisher={Wiley Online Library}
}

@article{escofier1998analyses,
  title={Analyses factorielles simples et multiples},
  author={Escofier, Brigitte and Pag{\`e}s, J{\'e}r{\^o}me},
  journal={Dunod, Paris},
  pages={284},
  year={1998}
}

@misc{linting2007nonparametric,
  title={Nonparametric inference in nonlinear principal components analysis: exploration and beyond [Tesis]. Leiden: Leiden University; 2007 [citado 28 Ene 2020]},
  author={Linting, M},
  year={2007}
}

@inproceedings{de1987nonlinear,
  title={Nonlinear multivariate analysis with optimal scaling},
  author={de Leeuw, Jan},
  booktitle={Develoments in Numerical Ecology},
  pages={157--187},
  year={1987},
  organization={Springer}
}

@article{zou2006sparse,
  title={Sparse principal component analysis},
  author={Zou, Hui and Hastie, Trevor and Tibshirani, Robert},
  journal={Journal of computational and graphical statistics},
  volume={15},
  number={2},
  pages={265--286},
  year={2006},
  publisher={Taylor \& Francis}
}

@article{candes2011robust,
  title={Robust principal component analysis?},
  author={Cand{\`e}s, Emmanuel J and Li, Xiaodong and Ma, Yi and Wright, John},
  journal={Journal of the ACM (JACM)},
  volume={58},
  number={3},
  pages={1--37},
  year={2011},
  publisher={ACM New York, NY, USA}
}

@article{liu2024kan,
  title={Kan: Kolmogorov-arnold networks},
  author={Liu, Ziming and Wang, Yixuan and Vaidya, Sachin and Ruehle, Fabian and Halverson, James and Solja{\v{c}}i{\'c}, Marin and Hou, Thomas Y and Tegmark, Max},
  journal={arXiv preprint arXiv:2404.19756},
  year={2024}
}

@article{pedregosa2011scikit,
  title={Scikit-learn: Machine learning in Python},
  author={Pedregosa, Fabian and Varoquaux, Ga{\"e}l and Gramfort, Alexandre and Michel, Vincent and Thirion, Bertrand and Grisel, Olivier and Blondel, Mathieu and Prettenhofer, Peter and Weiss, Ron and Dubourg, Vincent and others},
  journal={the Journal of machine Learning research},
  volume={12},
  pages={2825--2830},
  year={2011},
  publisher={JMLR. org}
}

@inproceedings{ribeiro2016should,
  title={" Why should i trust you?" Explaining the predictions of any classifier},
  author={Ribeiro, Marco Tulio and Singh, Sameer and Guestrin, Carlos},
  booktitle={Proceedings of the 22nd ACM SIGKDD international conference on knowledge discovery and data mining},
  pages={1135--1144},
  year={2016}
}

@inproceedings{virgolin2017scalable,
  title={Scalable genetic programming by gene-pool optimal mixing and input-space entropy-based building-block learning},
  author={Virgolin, Marco and Alderliesten, Tanja and Witteveen, Cees and Bosman, Peter AN},
  booktitle={Proceedings of the Genetic and Evolutionary Computation Conference},
  pages={1041--1048},
  year={2017}
}

\appendix

\section{Synthetic Datasets}
\label{appendix:synthetic_data}

\subsubsection{Nested circles}

The \textit{nested circles} dataset was generated using the \textit{make\_circles} function from the scikit-learn library \cite{scikit-learn} using the following parameter values: $n\_samples=1000$, $factor=0.1$, $noise=0.1$. This results in a small circle nested into a larger one (both centred around the origin) populated around the perimeter, to which some Gaussian noise $\epsilon \sim N(0, 0.1)$ is added.

\subsubsection{Nested spheres}

The \textit{nested spheres} dataset is similar to \textit{nested circles} as it composed of a smaller sphere contained inside a larger one, both centred around the origin. The data is sampled approximately uniformly on the surface of the spheres using the Fibonacci sphere algorithm \cite{gonzalez2010measurement}, after which some Gaussian noise $\epsilon \sim N(0, 0.1)$ is added.

\subsubsection{Alternating stripes}
\label{appendix:alternate_stripes_explanation}

The \textit{alternating stripes} dataset consists of 2-dimensional stripes with $X^{(1)} = n\pi + \epsilon$, where $n \in \{-4, -3, \ldots, 3, 4\}$ with $\epsilon \sim N(0, 0.1)$, and $X^{(2)}=\epsilon$. The idea is that with the periodicity in $X^{(1)}$ and $X^{(2)}$ there exist transformations of the variables that perfectly correlate.

\begin{figure*}[ht]
    \centering
    \begin{subfigure}{0.3\textwidth}
        \includegraphics[width=\textwidth]{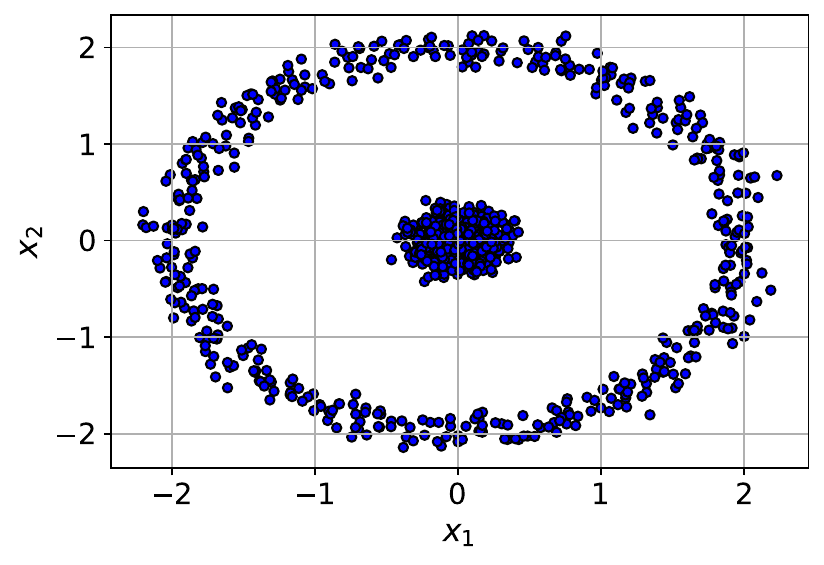}
        \caption{\textit{Nested circles}.}
        \label{fig:circle_first}
    \end{subfigure}
    \hfill
    \begin{subfigure}{0.3\textwidth}
        \includegraphics[width=\textwidth]{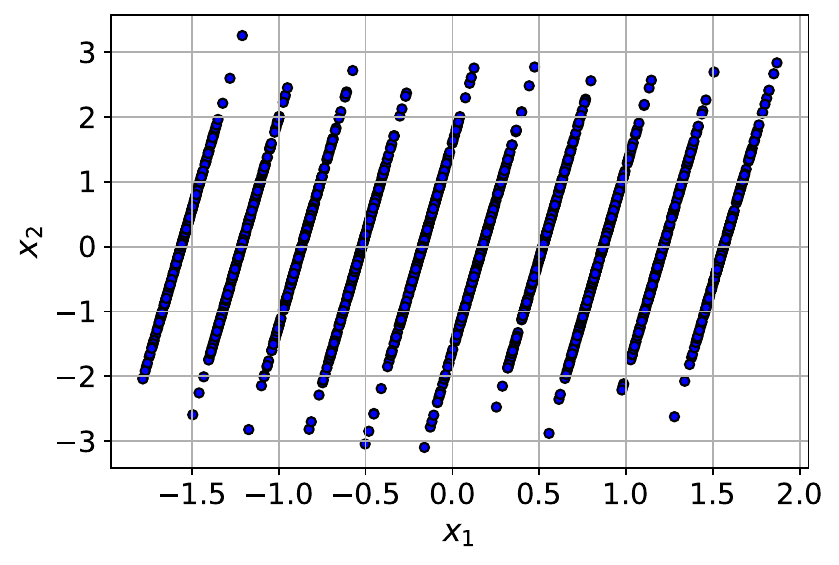}
        \caption{\textit{Alternating stripes}.}
        \label{fig:stripes_first}
    \end{subfigure}
    \hfill
    \begin{subfigure}{0.3\textwidth}
     \vspace{-1cm}
        \includegraphics[width=\textwidth]{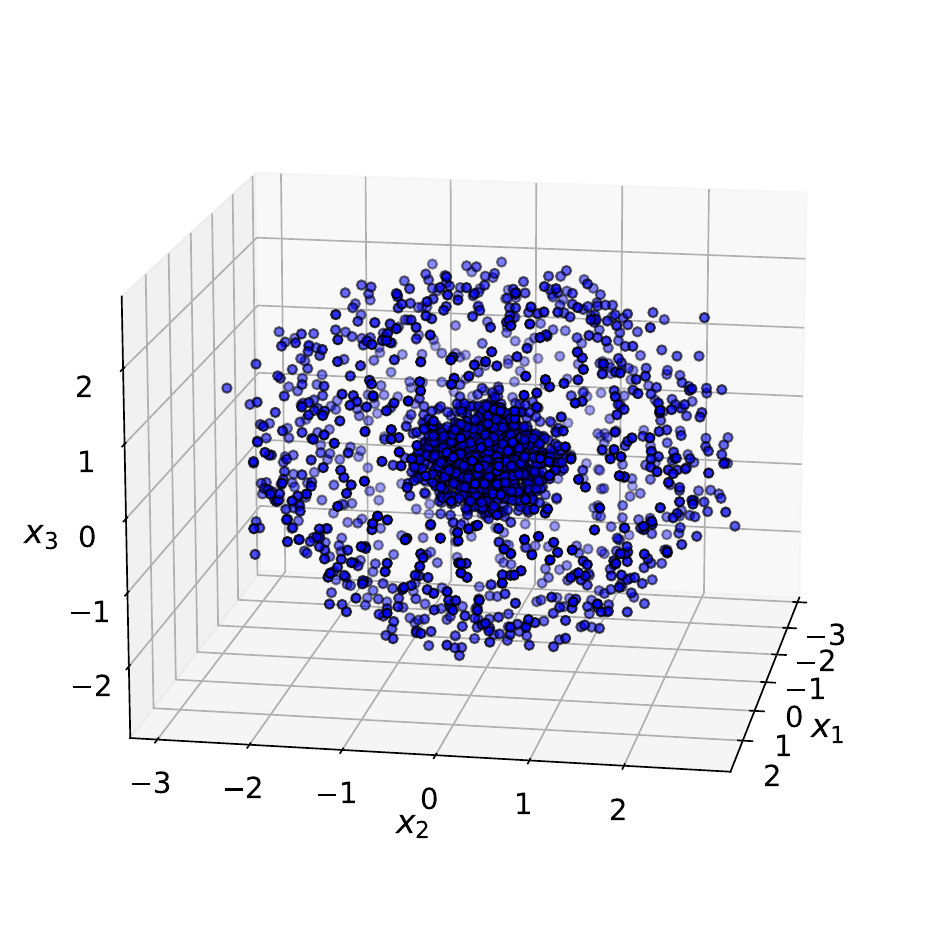}
        \vspace{-1.1cm}
        \caption{\textit{Nested spheres}.}
        \label{fig:sphere_first}
    \end{subfigure}
    \caption{Synthetic datasets.}
    \label{fig:synthetic_data}
\end{figure*}

\section{Genetic Programming}
\label{appendix:genetic_programming}

In GP, we start from a random population of individuals (represented as computational trees as shown in Figure \ref{fig:gp_diagram} below) which are then evolved (optimized) iteratively for a given number of generations, using various operations called genetic operators. At each generation, solutions are randomly modified and combined with each other where only the chosen solutions are passed on to the next generation. Then, at the end of the optimization procedure, the solution with the best objective value is selected and applied to our problem. 

For each generation, a basic tournament selection process is used to choose the individuals that will act as parents to produce offspring via the genetic operators. Each parent can be subject to up to three different genetic operators: (i) \textit{crossover}, (ii) \textit{subtree mutation}, and (iii) \textit{one-point mutation}, with probabilities of $p_c$, $p_s$, and $p_o$, respectively, for which the values are given in Table \ref{tab:gp_parameters}. Since the transformations $\Phi_{j}, j=1,\ldots,p$, are strictly functions of their associated variable $X^{(j)}$ (see Figure \ref{fig:transformations}), the crossover operator is only applied to transformations that act on the same variable. 

The population at the first generation is initialized using the popular ramped half-and-half method~\cite{koza1992genetic}, and random constants are drawn from $N(0, 1)$. The terminal set (i.e., the set of all possible leaf nodes) for each transformation $\Phi_{j}$ is given by $\mathcal{T}_{j} = \{X^{(j)}, \mathcal{R}\}$, where $\mathcal{R}$ represents the set of random constants. The function set (i.e., the set of possible non-leaf nodes) is
\[
F = \{-, +, \times, \textrm{cos}, \textrm{sin}\}.
\]

\begin{table}[ht!]
\centering
\caption{GP parameters and their values.}
\label{tab:gp_parameters}
\begin{tabular}{lc}
\toprule  
\textbf{Parameter} & \textbf{Value} \\
\midrule
Crossover rate ($p_c$) & 0.8 \\
Subtree mutation rate ($p_s$) & 0.2 \\
Operator mutation rate ($p_o$) & 0.2 \\
Population size ($P$) & 1000 \\
Generations & 100 \\
Tree depth: (Min, Max) & (2, 7) \\
Tournament Size & 7 \\
\bottomrule
\end{tabular}
\end{table}

\begin{figure}[H]
\centering
\includegraphics[width=0.5\linewidth]{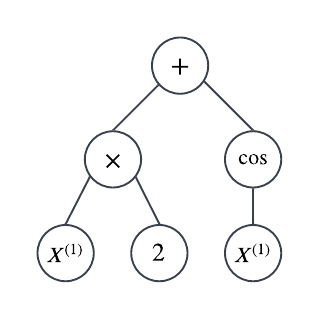}
\caption{Example of a computational tree of size 6 representing $\Phi_{1} = 2 \times X^{(1)} + \textrm{cos}(X^{(1)})$.}
\label{fig:gp_diagram}
\end{figure}

\section{kPCA results visualization}
\label{appendix:kpca_viz_synthetic}

In this section, we display the plots for each of the original synthetic datasets projected onto the first two principal components computed using various kernel functions.

\section{Evolution Strategies}

\subsection{Neural Network Architecture}
\label{appendix:nn}

In Table \ref{tab:nn_pameters}, we describe the chosen architecture of the neural networks for each of the variable types: numerical and categorical. For both types of variables, we choose a simple, over-parametrized architecture of 2 hidden layers of 64 neurons with batch normalization.

\begin{table}[H]
\centering
\caption{Neural network hyperparameters.}
\label{tab:nn_pameters}
\begin{tabular}{lccc}
\toprule  
\textbf{Variable type} & \textbf{\# hidden layers} & \textbf{\# neurons} & \textbf{Activation} \\
\midrule
Numerical & 1 & 64 & ReLU \\
\bottomrule
\end{tabular}
\end{table}

\subsection{Evolution Strategies}
\label{appendix:es}

In this section, we describe the training procedure used in this paper. Essentially, using ES is analogous to training a neural network with standard gradient descent except that the gradient step is replaced by an ES step.

\begin{table}[h!]
\centering
\caption{ES hyperparameter values.}
\label{tab:es}
\begin{tabular}{lcccc}
\toprule  
\textbf{Population size $P$} & \textbf{Noise $\sigma$} & \textbf{Learning rate $\alpha$} & \textbf{Batch size} \\
\midrule
200 & $10^{-2}$ & $10^{-2}$ & 128 \\
\bottomrule
\end{tabular}
\end{table}


\clearpage
\label{appendix:double_column_figures}

\begin{figure*}[t!]
\centering
\begin{subfigure}{0.3\textwidth}
\includegraphics[width=\textwidth]{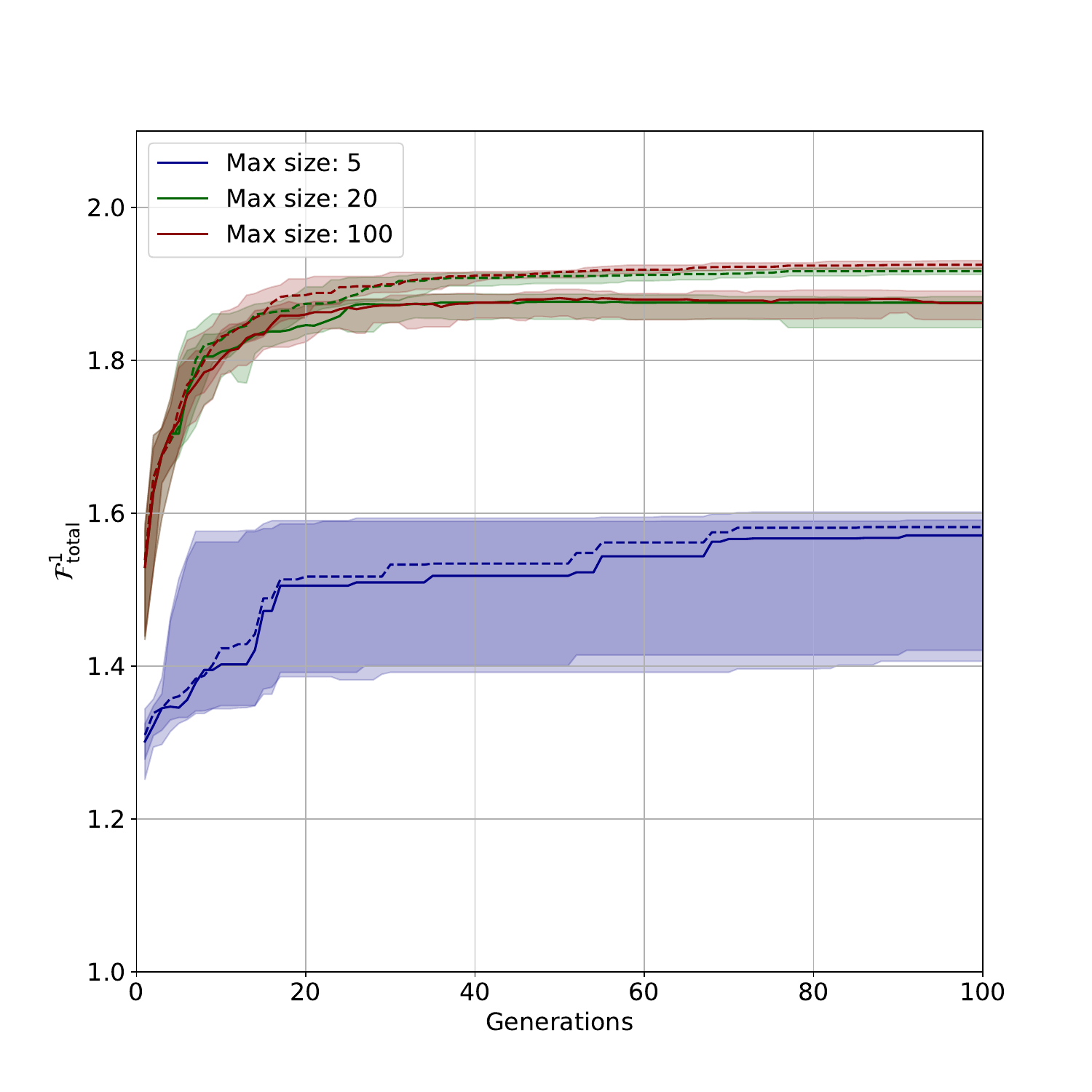}
\caption{Nested circles dataset.}
\end{subfigure}
\hfill
\begin{subfigure}{0.3\textwidth}
\includegraphics[width=\textwidth]{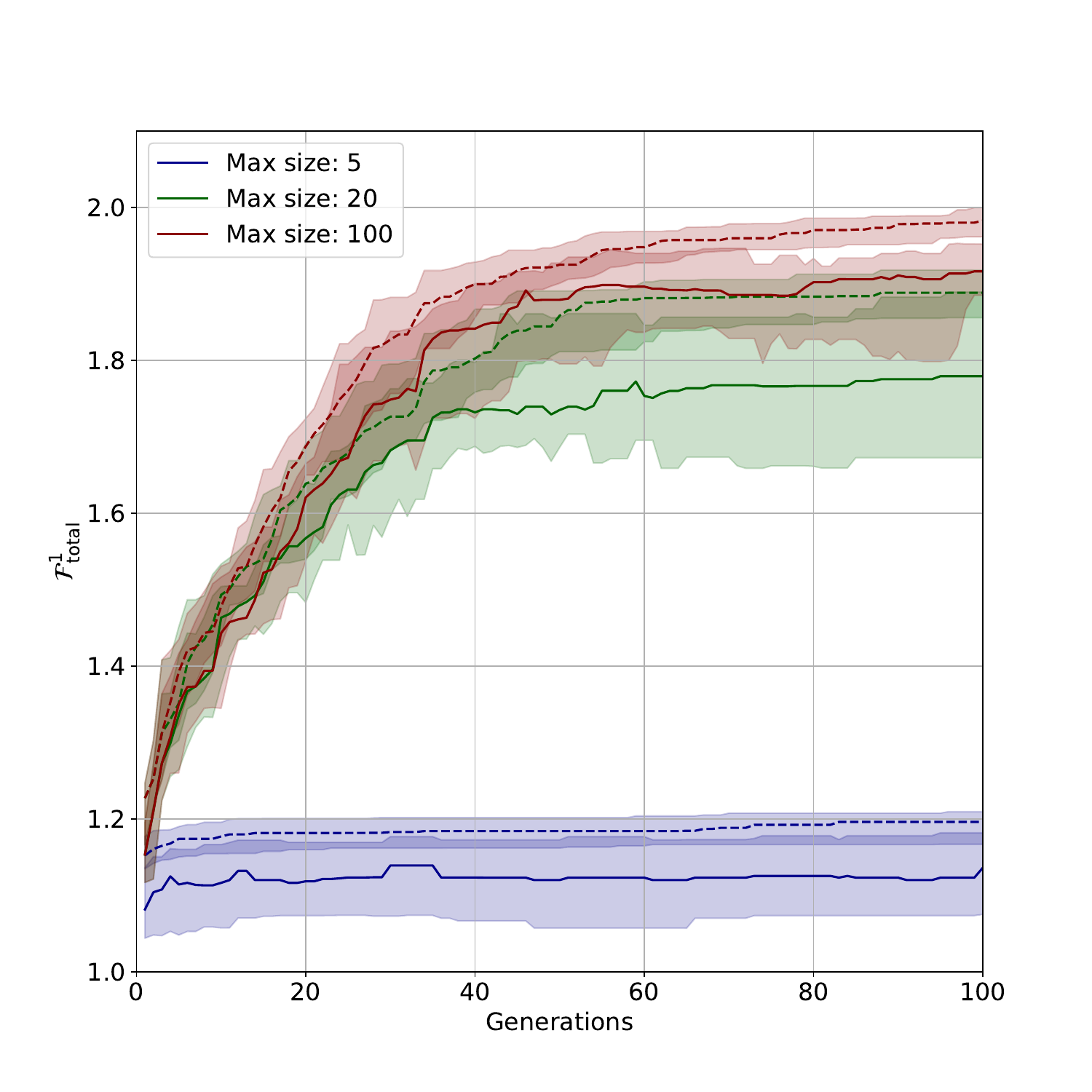}
\caption{Nested spheres dataset.}
\end{subfigure}
\hfill
\begin{subfigure}{0.3\textwidth}
\includegraphics[width=\textwidth]{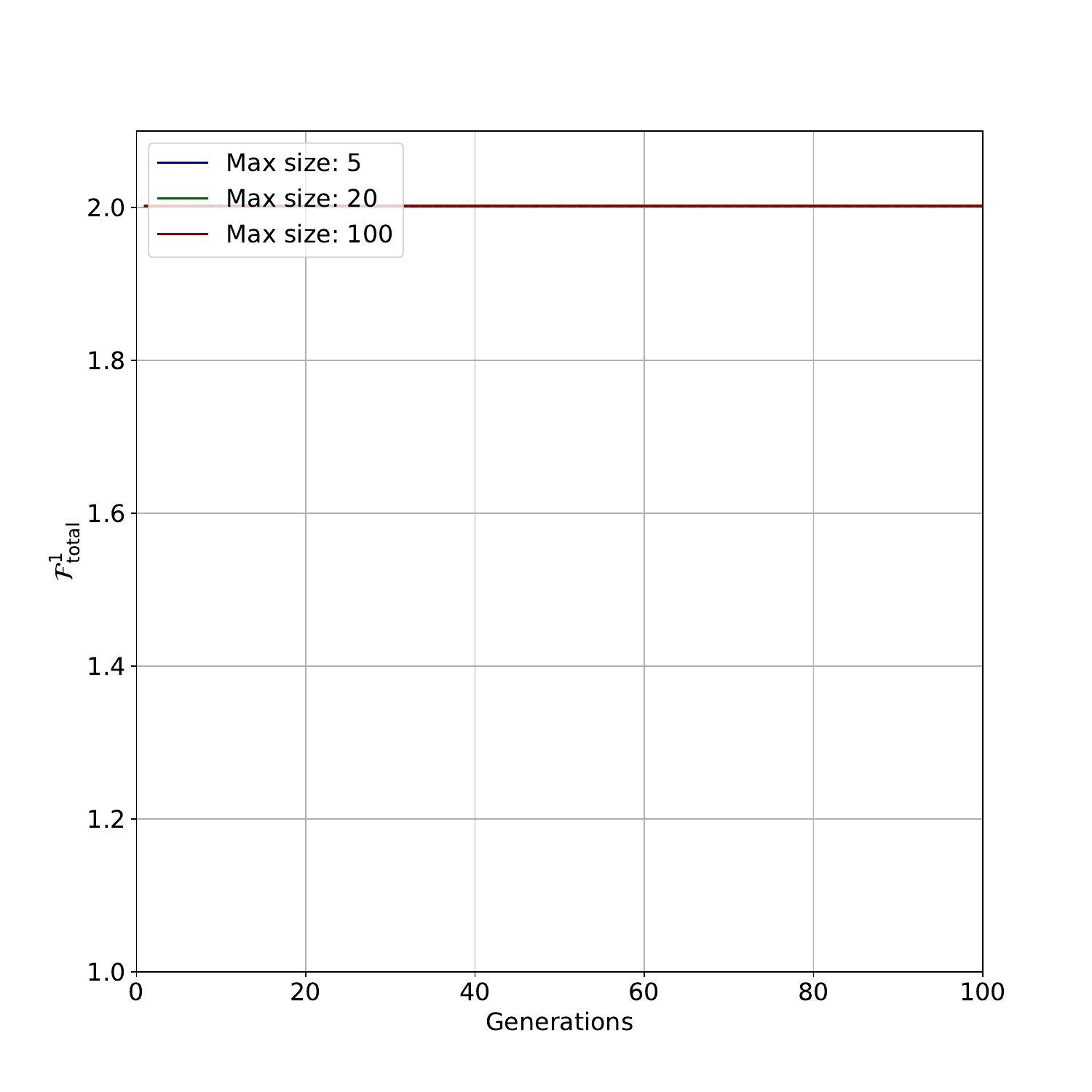}
\caption{Alternating stripes dataset.}
\end{subfigure}
\caption{Fitness function $\mathcal{F}_{\textrm{global}}^{1}$ on the validation (solid) and training (dashed) sets over 100 generations.}
\label{fig:generations}
\end{figure*}

\begin{figure*}[t!]
\centering
\begin{subfigure}{0.24\textwidth}
\includegraphics[width=\textwidth]{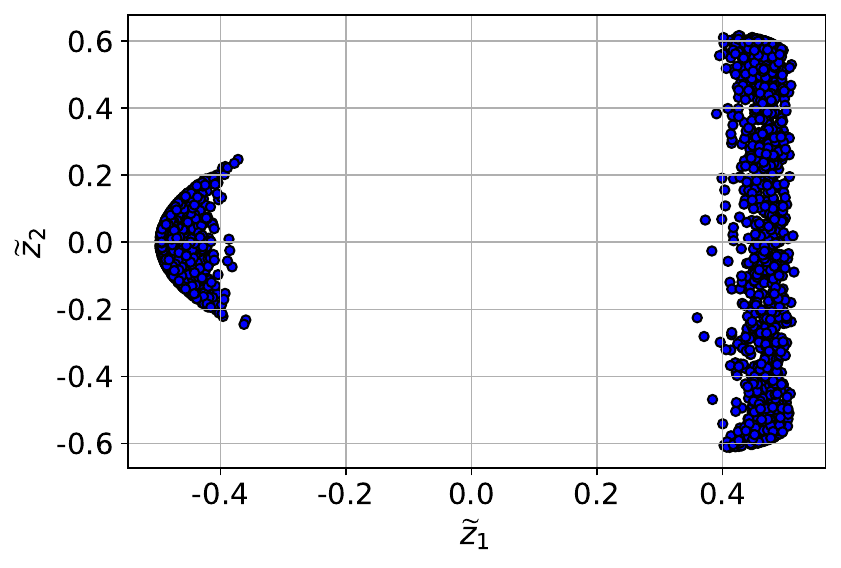}
\caption{RBF kernel.}
\end{subfigure}
\hfill
\begin{subfigure}{0.24\textwidth}
\includegraphics[width=\textwidth]{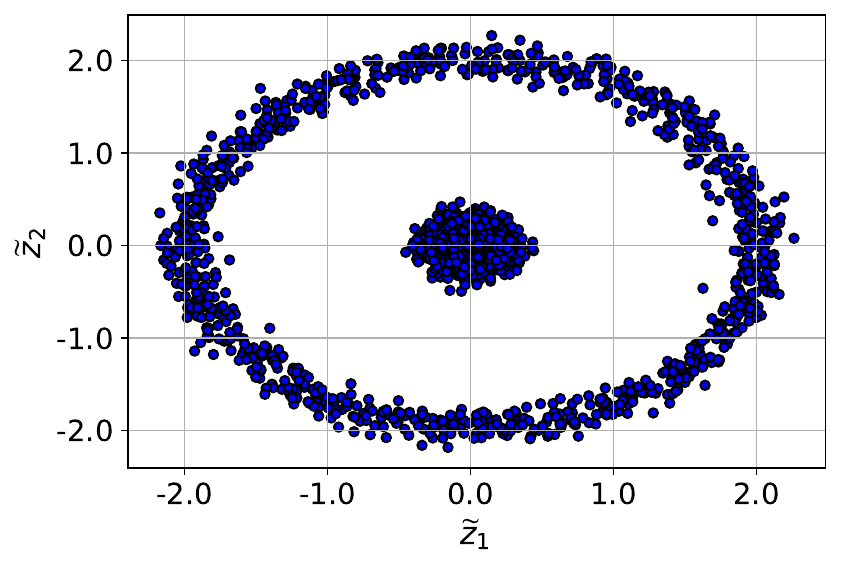}
\caption{Polynomial kernel (d=2).}
\end{subfigure}
\hfill
\begin{subfigure}{0.24\textwidth}
\includegraphics[width=\textwidth]{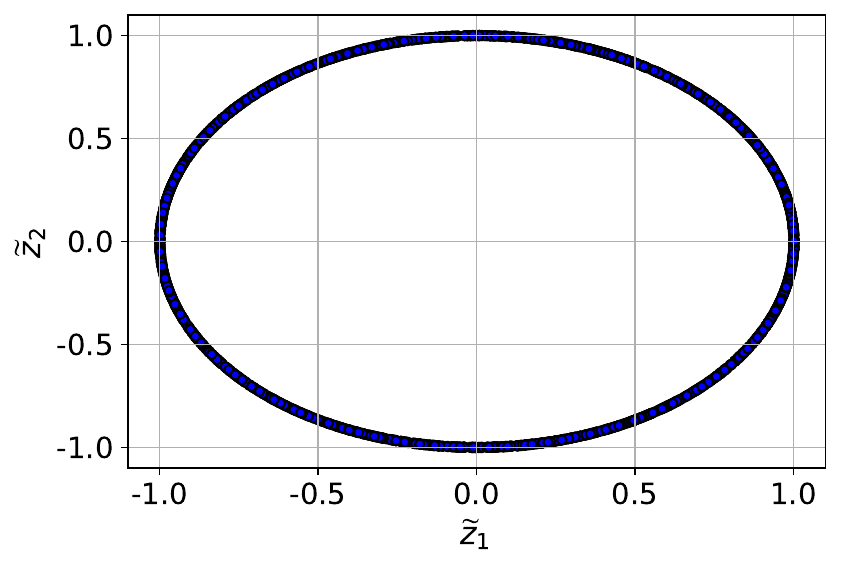}
\caption{Cosine kernel.}
\end{subfigure}
\hfill
\begin{subfigure}{0.24\textwidth}
\includegraphics[width=\textwidth]{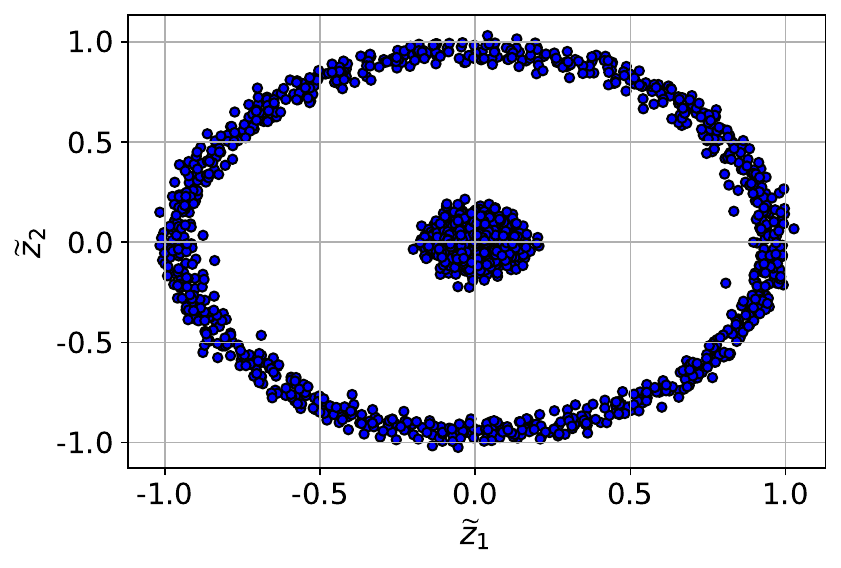}
\caption{Sigmoid kernel.}
\end{subfigure}
\caption{Projection of the nested circles data using kPCA.}
\label{fig:circles_kpca}
\end{figure*}

\begin{figure*}[t!]
\centering
\begin{subfigure}{0.24\textwidth}
\includegraphics[width=\textwidth]{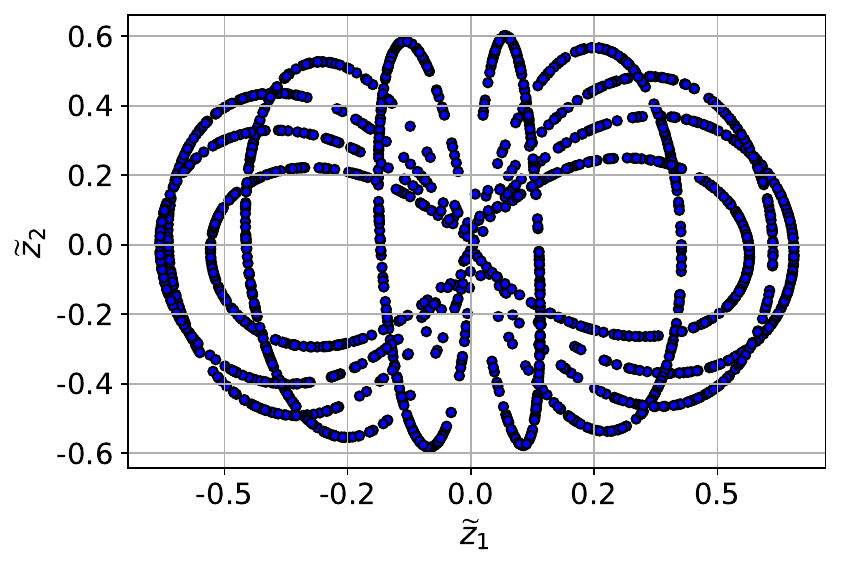}
\caption{RBF kernel.}
\end{subfigure}
\hfill
\begin{subfigure}{0.24\textwidth}
\includegraphics[width=\textwidth]{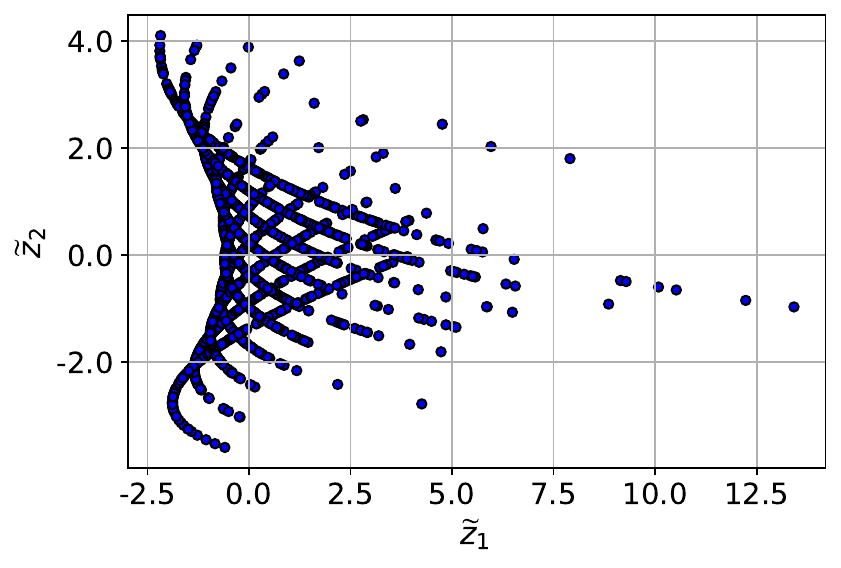}
\caption{Polynomial kernel (d=2).}
\end{subfigure}
\hfill
\begin{subfigure}{0.24\textwidth}
\includegraphics[width=\textwidth]{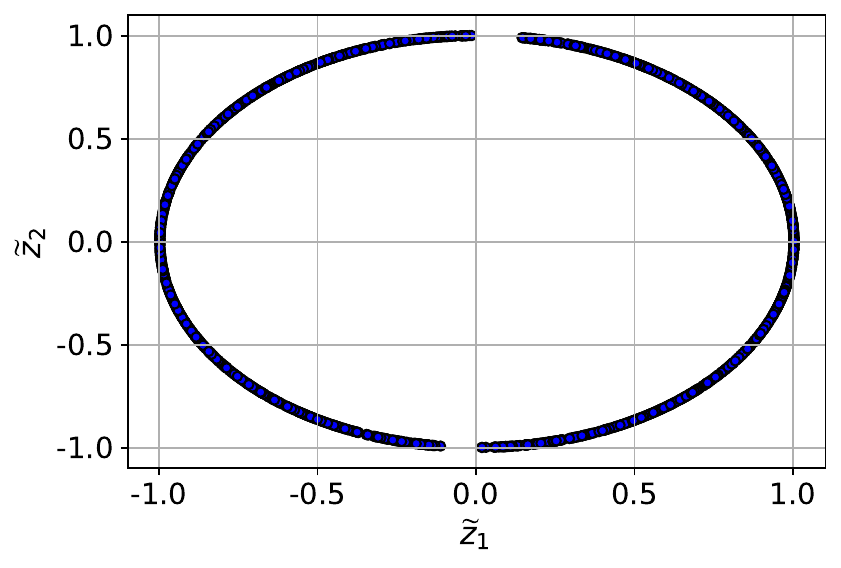}
\caption{Cosine kernel.}
\end{subfigure}
\hfill
\begin{subfigure}{0.24\textwidth}
\includegraphics[width=\textwidth]{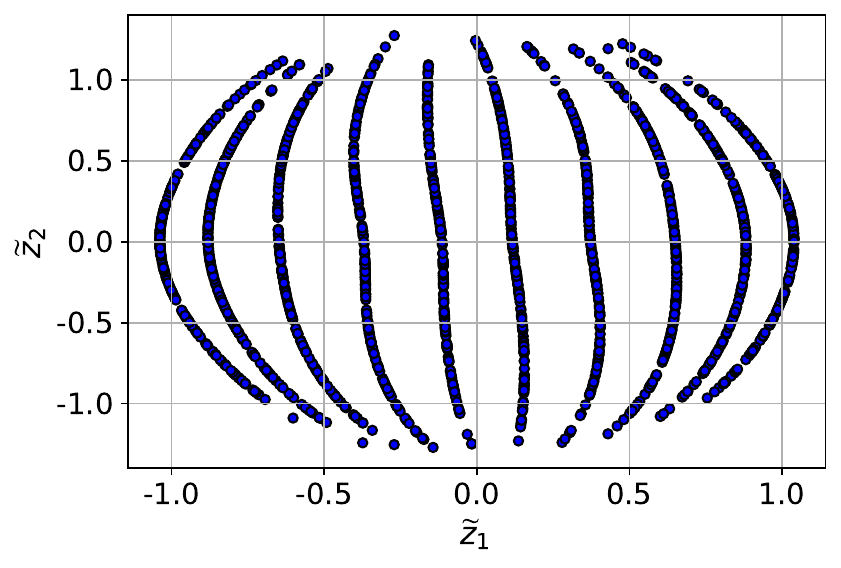}
\caption{Sigmoid kernel.}
\end{subfigure}
\caption{Projection of the alternating stripes data using kPCA.}
\label{fig:alternate_stripes_kpca}
\end{figure*}

\begin{figure*}[t!]
\centering
\begin{subfigure}{0.24\textwidth}
\includegraphics[width=\textwidth]{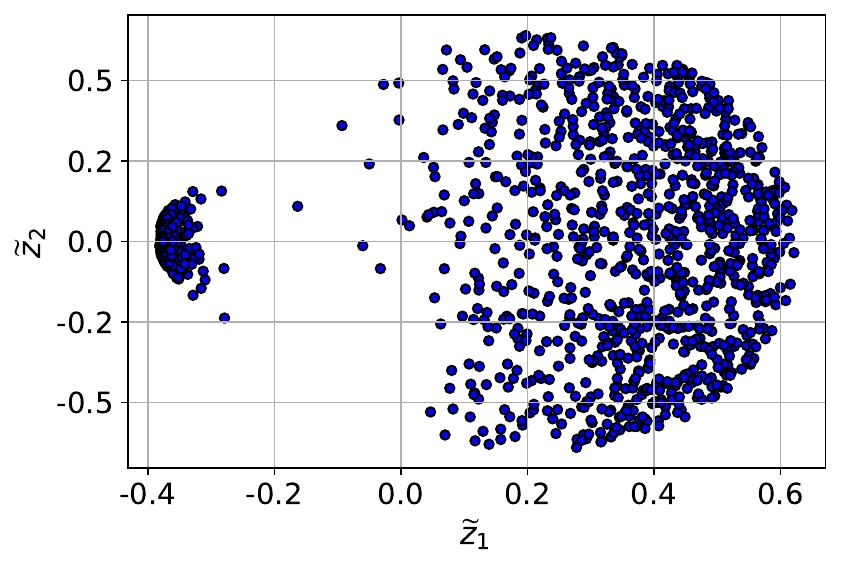}
\caption{RBF kernel.}
\end{subfigure}
\hfill
\begin{subfigure}{0.24\textwidth}
\includegraphics[width=\textwidth]{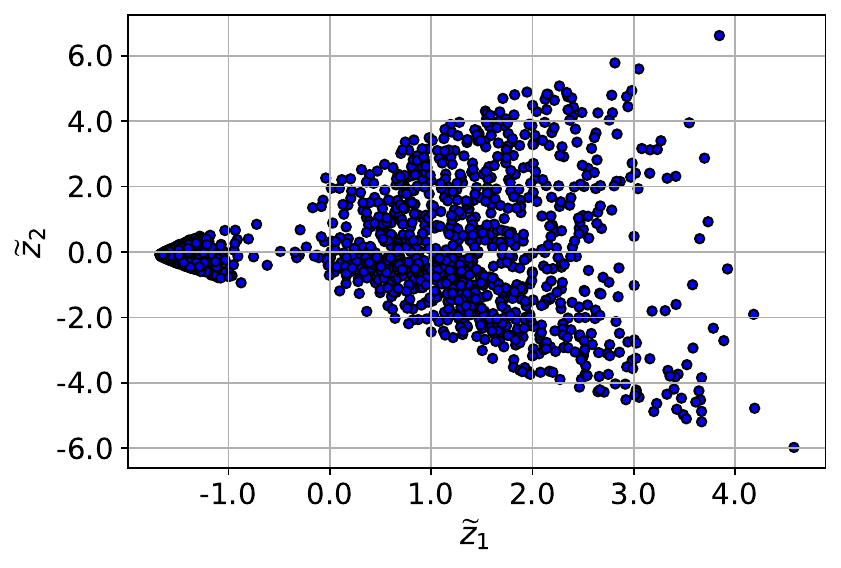}
\caption{Polynomial kernel (d=2).}
\end{subfigure}
\hfill
\begin{subfigure}{0.24\textwidth}
\includegraphics[width=\textwidth]{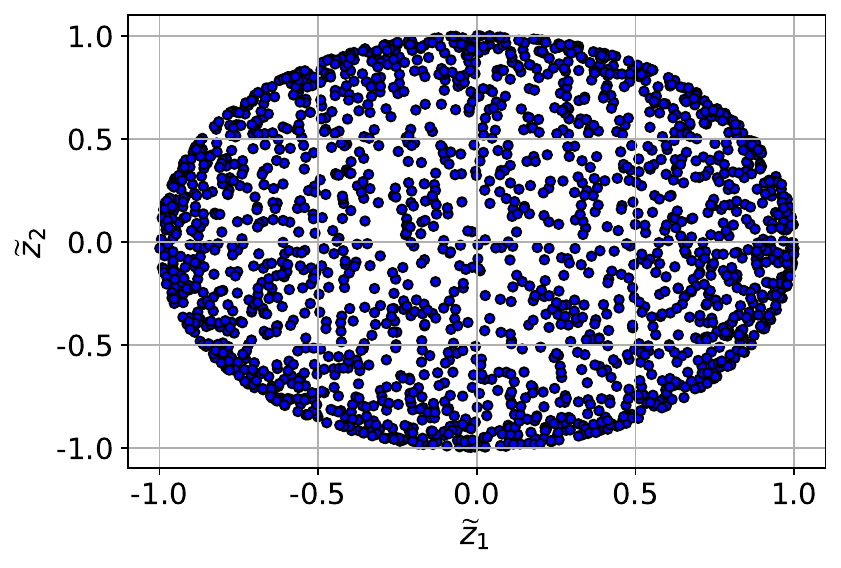}
\caption{Cosine kernel.}
\end{subfigure}
\hfill
\begin{subfigure}{0.24\textwidth}
\includegraphics[width=\textwidth]{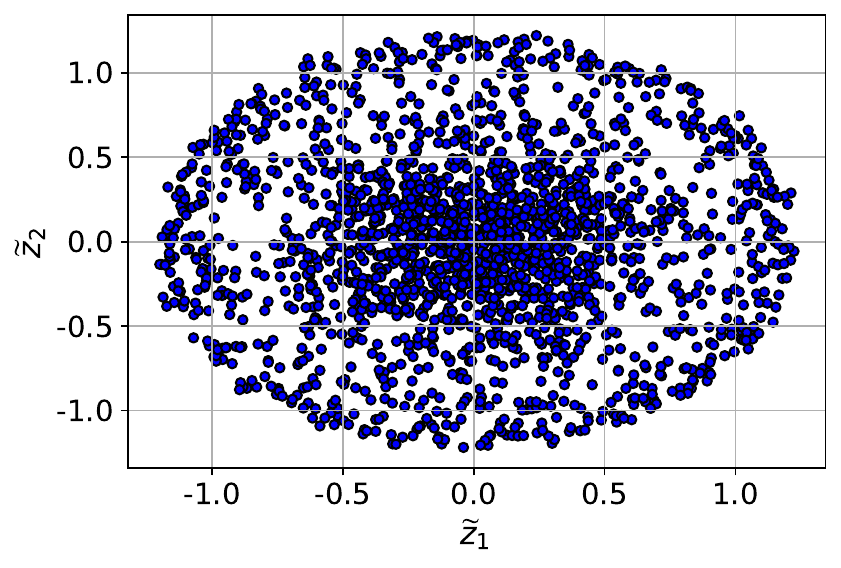}
\caption{Sigmoid kernel.}
\end{subfigure}
\caption{Projection of the nested spheres data using kPCA.}
\label{fig:spheres_kpca}
\end{figure*}

\end{document}